\documentclass{article}

     \PassOptionsToPackage{authoryear,semicolon,square}{natbib}

\usepackage[preprint]{neurips_2025}

\usepackage[utf8]{inputenc} %
\usepackage[T1]{fontenc}    %
\usepackage{xr-hyper,hyperref}       %
\usepackage{url}            %
\usepackage{booktabs}       %
\usepackage{amsfonts}       %
\usepackage{nicefrac}       %
\usepackage{microtype}      %
\usepackage{xcolor}         %
\usepackage{amsmath}
\usepackage{amssymb}
\usepackage{amsthm}
\usepackage[capitalise]{cleveref}
\usepackage{comment}
\usepackage[rightcaption]{sidecap}

\usepackage{svg}
\usepackage{svg-extract}
\graphicspath{{svg-inkscape/}}

\newcommand{\kf}{k^f}
\newcommand{\ks}{k^{\sigma}}
\newcommand{\kw}{k^w}
\newcommand{\ksum}{\kf + \ks}

\newcommand{\Hk}{\mathcal{H}_k}
\newcommand{\Hkf}{\mathcal{H}_{\kf}}
\newcommand{\Hksum}{\mathcal{H}_{\ksum}}
\newcommand{\Hkw}{\mathcal{H}_{\kw}}
\newcommand{\Hks}{\mathcal{H}_{\ks}}

\newcommand{\fmu}[1][\sigma]{f^\mu_{#1}}
\newcommand{\wmu}[1][\sigma]{w^\mu_{#1}}
\newcommand{\gmu}[1][\sigma]{g^\mu_{#1}}
\newcommand{\thetamu}[1][\sigma]{\theta^\mu_{#1}}
\newcommand{\fopt}{f^\star}
\newcommand{\wopt}{w^\star}

\newcommand{\thetaopt}{\theta^\star}
\newcommand{\ftr}{f^{\mathrm{tr}}}
\newcommand{\wtr}{w^{\mathrm{tr}}}
\newcommand{\thetatr}{\theta^{\mathrm{tr}}}
\newcommand{\ftrint}{f^{\mathrm{tr,int}}}
\newcommand{\wtrint}{w^{\mathrm{tr,int}}}

\newtheorem{theorem}{Theorem}
\newtheorem{lemma}{Lemma}
\newtheorem{corollary}{Corollary}
\newtheorem{proposition}{Proposition}

\theoremstyle{definition}

\newtheorem{assumption}{Assumption}

\theoremstyle{remark}
\newtheorem*{remark}{Remark}

\crefname{theorem}{Theorem}{Theorems}
\crefname{proposition}{Proposition}{Propositions}
\crefname{corollary}{Corollary}{Corollaries}
\crefname{lemma}{Lemma}{Lemmas}
\crefname{assumption}{Assumption}{Assumptions}
\crefname{equation}{Eq.}{Eqs.}
\crefname{section}{Section}{Sections}
\crefname{appendix}{Appendix}{Appendices}

\makeatletter
\newcommand*{\addFileDependency}[1]{%
  \typeout{(#1)}
  \@addtofilelist{#1}
  \IfFileExists{#1}{}{\typeout{No file #1.}}
}
\makeatother

\newcommand*{\myexternaldocument}[1]{%
    \externaldocument{#1}%
    \addFileDependency{#1.tex}%
    \addFileDependency{#1.aux}%
}

\myexternaldocument{supplement}

\title{Optimal kernel regression bounds \\under energy-bounded noise}

\newif\ifabstractonly
\abstractonlyfalse

\newif\ifchecklist
\checklistfalse

\author{%
    Amon Lahr \\
    ETH Zurich \\
    \texttt{amlahr@ethz.ch}
    \And
    Johannes Köhler\thanks{Both co-authors contributed equally; their ordering is alphabetical.} \\
    ETH Zurich \\
    \texttt{jkoehle@ethz.ch}
    \And
    Anna Scampicchio$^{*}$ \\
    ETH Zurich \\
    \texttt{ascampicc@ethz.ch}
    \And
    Melanie N. Zeilinger
    \\
  ETH Zurich \\
  \texttt{mzeilinger@ethz.ch}
}

\begin{document}

\maketitle

\begin{abstract}

Non-conservative uncertainty bounds 
are
key for both assessing 
an estimation algorithm's
accuracy
and in view of downstream tasks, such as its deployment in safety-critical contexts. 
In this paper, we 
derive
a
tight, 
non-asymptotic 
uncertainty bound for kernel-based estimation,
which
can also 
handle correlated 
noise sequences. 
Its computation 
relies on
a mild norm-boundedness assumption on the unknown function and the noise,
returning the worst-case function realization within the hypothesis class at an arbitrary query input location.
The value of this
function is shown
to be given
in terms of 
the posterior mean and covariance of a Gaussian process
for an optimal choice of the
measurement 
noise covariance.
By rigorously analyzing the proposed approach and comparing it with other results in the literature, we show its effectiveness in returning tight and easy-to-compute bounds for kernel-based estimates.
\end{abstract}

\ifabstractonly

\else

\setcounter{footnote}{0}

\section{Introduction}
Many problems in machine learning can be phrased in terms %
of estimating an unknown (continuous) function from a finite set of noisy data. A popular, non-parametric technique to perform such a task and return \emph{point-wise} estimates is given by the class of kernel-based methods~\citep{wahba_spline_1990,scholkopf_learning_2001,suykens_least_2002,shawe-taylor_kernel_2004,steinwart_support_2008}. %
Complementing such estimates with 
non-conservative and non-asymptotic uncertainty bounds %
enables
evaluating their reliability, 
for example, 
in view of deploying Bayesian optimization~\citep{berkenkamp_bayesian_2023,sui_stagewise_2018} or model-based reinforcement learning~\citep{kuss_gaussian_2003,chua_deep_2018} to safety-critical systems. 

Classical uncertainty bounds for kernel-based methods 
have been developed in statistical learning theory~\citep{cucker_mathematical_2002,cucker_learning_2007,guo_concentration_2013,lecue_regularization_2017,ziemann_learning_2024}. 
However,  these results are mostly aimed at characterizing the learning rate of the kernel-based algorithm and they tend to be difficult to apply in practice, being overly conservative or even depending on the unknown function to be estimated. %
Another viewpoint is given by 
Gaussian process (GP) regression~\citep{rasmussen_gaussian_2006},
a kernel-based method that 
is naturally endowed
with an uncertainty quantification mechanism.
However, 
closed-form uncertainty bounds are only available
when 
assuming independent and Gaussian-distributed variables~\citep{lederer_uniform_2019}, and their computation in other cases is non-trivial~\citep{gilks_markov_1995}. 
To address this issue, high-probability and non-asymptotic uncertainty bounds have been derived by~\cite{srinivas_information-theoretic_2012,abbasi-yadkori_online_2013,burnaev_efficiency_2014,fiedler_practical_2021,baggio_bayesian_2022,molodchyk_towards_2025}, phrasing the problem as estimation in Reproducing Kernel Hilbert Spaces~(RKHSs)~\citep{aronszajn_theory_1950,berlinet_reproducing_2004}. Yet, these bounds still heavily rely on the (conditional) independence of the noise sequence, which can be hard to satisfy in practice. 
This difficulty can be circumvented by %
leveraging 
an assumed bound on the noise~\citep{maddalena_deterministic_2021,reed_error_2025,scharnhorst_robust_2023} -- however, these results 
tend to be
conservative %
or rely on solving a computationally intensive, constrained optimization problem to evaluate the uncertainty bound.%

\paragraph*{Contribution} In this paper, we propose a novel non-asymptotic uncertainty bound for kernel-based estimation assuming a general bound on the noise energy.
In particular, 
for each query input location,
the proposed bound
\emph{exactly} characterizes the worst-case latent function within the given hypothesis class. 
The obtained uncertainty bound has the same structure as the high-probability uncertainty bounds from GP regression~\citep{srinivas_information-theoretic_2012,abbasi-yadkori_online_2013,fiedler_practical_2021,molodchyk_towards_2025}, 
but with a measurement noise covariance~$\sigma^2$ that depends on the test input location. 
Furthermore, we show that the derived bound recovers results from kernel interpolation~\citep{weinberger_optimal_1959,wendland_scattered_2004} and linear regression~\citep{fogel_system_1979} as special cases. 
Finally, we contrast the proposed robust treatment 
to
existing bounds for GP regression.

\paragraph*{Notation}
The matrix \mbox{$I_n \in \mathbb{R}^{n \times n}$} denotes the identity matrix of dimension $n$.
For a symmetric positive-semidefinite matrix \mbox{$A\in\mathbb{R}^{n\times n}$}, $A^{1/2}$ denotes the (positive-semidefinite) symmetric matrix square root, i.e., \mbox{$A^{1/2}A^{1/2}=A$}, and $\| x \|_A^2 \doteq x^\top A x$ denotes the weighted Euclidean norm of a vector~\mbox{$x \in \mathbb{R}^n$}. 
The Dirac delta function is denoted by \mbox{$\delta: \mathbb{R}^{n_x} \rightarrow \mathbb{R}$}, with \mbox{$\delta(x) = 1$} for \mbox{$x = 0$} and \mbox{$\delta(x) = 0$} otherwise.
Superscripts $f$ and $w$ will refer to the latent function and the noise, respectively. Accordingly, the kernel function is denoted by \mbox{$k^{\square}\colon \mathbb{R}^{n_x} \times \mathbb{R}^{n_x} \to \mathbb{R}$}, with \mbox{$\square \in \{f,w\}$}, and the associated RKHS is denoted by \mbox{$\mathcal{H}_{k^{\square}}$}. For two arbitrary ordered sets of indices \mbox{$\mathbb{I},\mathbb{J}\subseteq\mathbb{N}$}, the matrix \mbox{$K^{\square}_{\mathbb{I},\mathbb{J}} \doteq [k^{\square}(x_i,x_j)]_{i \in \mathbb{I}, j \in \mathbb{J}}$} is the Gram matrix collecting the evaluations of the kernel function $k^{\square}$ at pairs of input locations $x_i$, $x_j$,
with \mbox{$i \in \mathbb{I}$} and \mbox{$j \in \mathbb{J}$}.
We denote by \mbox{$1 : N \doteq \{1,\dots,N\}\subseteq\mathbb{N}$} the set of indices for the training data points, while we use $x_{N+1}$ to represent the arbitrary test input. For instance, \mbox{$K^f_{N+1,1:N} \in \mathbb{R}^{1 \times N}$} corresponds to $\begin{bmatrix}
  \kf(x_{N+1},x_1), \ldots, \kf(x_{N+1},x_N)
\end{bmatrix}$ and can be interpreted as the covariance matrix between the test- and training-input locations.

\section{Problem set-up}
We consider the problem of estimating an unknown latent function \mbox{$\ftr \colon \mathcal{X} \to \mathbb{R}$}, with $\mathcal{X} \subseteq \mathbb{R}^{n_x}$ from noisy measurements
\begin{align}
  \label{eq:setup_data}
  y_i  &= \ftr(x_i) + \wtr(x_i), \quad i = 1, \ldots, N,
\end{align}
collected at known training input locations \mbox{$x_i \in \mathcal{X}$}. %
Our goal is to compute worst-case uncertainty bounds around the latent function $\ftr$ given the observed data set \mbox{$\mathcal{D} \doteq \{ (x_i, y_i) \}_{i=1}^{N}$}, which is subject to the unknown
noise~\mbox{$\wtr \colon \mathcal{X} \to \mathbb{R}^{}$}. 
We 
phrase 
the problem in the framework of estimation in RKHSs~\citep{aronszajn_theory_1950,berlinet_reproducing_2004}, and we model both the 
latent function~$\ftr$ 
and the 
noise~$\wtr$ 
as elements of an RKHS with a known kernel and a bound on their RKHS norm. %
\begin{assumption}
  \label{ass:RKHS_norm_fw}
  The unknown latent and noise functions are respective elements of the RKHSs corresponding to the 
  positive-semidefinite kernel \mbox{$\kf \colon \mathcal{X} \times \mathcal{X} \rightarrow \mathbb{R}_{\geq 0}$}
  and the positive-definite kernel \mbox{$\kw \colon \mathcal{X} \times \mathcal{X} \rightarrow \mathbb{R}_{\geq 0}$}, 
  where both $\kf$ and $\kw$ are uniformly bounded.
  There exist 
  known constants $\Gamma_f$, $\Gamma_w>0$ 
  strictly bounding their respective RKHS norms,
  i.e.,
  \mbox{$\| \ftr \|_{\Hkf}^2 < \Gamma_f^2$} and 
  \mbox{$\| \wtr \|_{\Hkw}^2 < \Gamma_w^2$}.
\end{assumption}

Characterizing boundedness of the noise using a kernel~$\kw$ and an RKHS-norm bound~$\Gamma_w^2$ provides a very general description and can model various scenarios: for instance, it captures the setting in which the noise sequence has bounded energy, as we elucidate in~\cref{sec:discussion_noise_RKHS}. Additionally, modeling noise as a deterministic quantity allows us to by-pass additional assumptions on the distribution or independence of the noise, latent function or input locations.

Since we model both the latent function and the noise as deterministic objects, there cannot be multiple output measurements at the same input location. Hence, we consider distinct inputs. 
\begin{assumption}
  \label{ass:distinct_input_locations}
  The training input locations in \mbox{$\mathbb{X} \doteq \{ x_1, \ldots x_N \}$} are pairwise distinct, i.e., \mbox{$x_i \neq x_j$} for all $i,j=1,\ldots,N$ and $i \neq j$.
\end{assumption}

\section{Kernel regression bounds for energy-bounded noise}\label{sec:main_result}
In the following, we present the main result of our paper, 
determining tight point-wise uncertainty bounds \mbox{$\underline{f}(x_{N+1}) \leq \ftr(x_{N+1}) \leq \overline{f}(x_{N+1})$} for the value of the latent function at an arbitrary test point \mbox{$x_{N+1} \in \mathcal{X}$}. This task can be formulated as an infinite-dimensional optimization problem, taking the bounded-RKHS-norm assumption into account.
The optimal upper bound $\overline{f}(x_{N+1})$ is defined as
\begin{subequations}
  \label{eq:sup_infdim_opt}
  \begin{align}
    \overline{f}(x_{N+1}) = \sup_{\substack{
      f \in \Hkf, \\
      w \in \Hkw \\ %
    }} 
    \quad & f(x_{N+1}) \\
    \mathrm{s.t.} \quad 
    & f(x_i) + w(x_i) = y_i, \> i=1,\ldots,N, \label{eq:sup_infdim_opt_data} \\
    & \|f\|_{\Hkf}^2 \leq \Gamma_f^2, \label{eq:sup_infdim_opt_bnd_f} \\
    & \|w\|_{\Hkw}^2 \leq \Gamma_w^2. \label{eq:sup_infdim_opt_bnd_w}
  \end{align}
\end{subequations}
Analogously, the optimal lower bound is given by:
\begin{subequations}
  \label{eq:inf_infdim_opt}
  \begin{align}
    \underline{f}(x_{N+1}) = \inf_{\substack{
      f \in \Hkf, \\
      w \in \Hkw \\ %
    }} 
    \quad & f(x_{N+1}) \\
    \mathrm{s.t.} \quad 
    & \text{\eqref{eq:sup_infdim_opt_data}-\eqref{eq:sup_infdim_opt_bnd_w}}.
  \end{align}
\end{subequations}
In the following, we focus on computing the upper bound $\overline{f}(x_{N+1})$; for the lower bound, the presented results 
in \cref{sec:relaxed_solution,sec:recovering_optimal_solution}
are analogously derived in~\cref{sec:app_relaxed_proof,sec:app_optimal_proof}, respectively.

Stated as in~\eqref{eq:sup_infdim_opt}, the optimization problem is infinite-dimensional and is not directly tractable. %
Our key result, presented in the remainder of the section, consists in finding an exact reformulation of the constrained, infinite-dimensional problem~\eqref{eq:sup_infdim_opt} as a scalar, unconstrained one.
The solution of the latter
at an arbitrary input location is expressed in terms of familiar quantities from Gaussian process regression, for an optimal choice of measurement noise covariance.
We present our derivation by first studying a relaxed formulation of this optimization problem in \cref{sec:relaxed_solution}. Then, in~\cref{sec:recovering_optimal_solution}, we discuss how to recover the optimal solution from the relaxed problem.

\subsection{Relaxed solution}
\label{sec:relaxed_solution}

The relaxed formulation of 
optimization problem~\eqref{eq:sup_infdim_opt}
considers the sum of the RKHS-norm constraints~\eqref{eq:sup_infdim_opt_bnd_f},~\eqref{eq:sup_infdim_opt_bnd_w}
instead of enforcing them 
individually:
\begin{subequations}
  \label{eq:sup_infdim_relax}
  \begin{align}
    \overline{f}^\sigma(x_{N+1}) = \sup_{\substack{
      f \in \Hkf, \\
      w \in \Hks \\ %
    }} 
    \quad & f(x_{N+1}) \\
    \mathrm{s.t.} \quad 
    & f(x_i) + w(x_i) = y_i, \> i=1,\ldots,N, \label{eq:sup_infdim_relax_data} \\
    & \|f\|_{\Hkf}^2 + \|w\|_{\Hks}^2 \leq \Gamma_f^2 + \frac{\Gamma_w^2}{\sigma^2}. \label{eq:sup_infdim_relax_bnd_fw}
  \end{align}
\end{subequations}
This problem uses a scaled noise kernel \mbox{$\ks(x,x') \doteq \sigma^2 \kw(x,x')$} 
with a constant output scale \mbox{$\sigma > 0$}, 
which
is key in relating the solution of the relaxed problem~\eqref{eq:sup_infdim_relax} to the original formulation~\eqref{eq:sup_infdim_opt}. Additionally, note that the scaling implies \mbox{$\|w\|_{\Hks}^2 = \|w\|_{\Hkw}^2/\sigma^2 \leq \Gamma_w^2/\sigma^2$}, as displayed in the constraint~\eqref{eq:sup_infdim_relax_bnd_fw}.

The 
bound~$\overline{f}^\sigma(x_{N+1})$
obtained from
the relaxed problem depends on the noise parameter~$\sigma$,
as the joint RKHS-norm constraint~\eqref{eq:sup_infdim_relax_bnd_fw} is given by a weighted sum of both original constraints~\eqref{eq:sup_infdim_opt_bnd_f} and~\eqref{eq:sup_infdim_opt_bnd_w}.
Any feasible solution of~\eqref{eq:sup_infdim_opt}
--
a tuple~$(f, w)$ of functions satisfying the constraints~\eqref{eq:sup_infdim_opt_data}-\eqref{eq:sup_infdim_opt_bnd_w}
--
is also a feasible solution for Problem~\eqref{eq:sup_infdim_relax}
for 
any
\mbox{$\sigma \in (0, \infty)$}.
Thus, 
\mbox{$\ftr(x_{N+1}) \leq \overline{f}(x_{N+1}) \leq \overline{f}^\sigma(x_{N+1})$},
i.e.,
the uncertainty envelope obtained by solving the relaxed problem~\eqref{eq:sup_infdim_relax} contains the 
one
obtained by solving the original problem~\eqref{eq:sup_infdim_opt}
for all test points~\mbox{$x_{N+1} \in \mathcal{X}$} and noise parameters~\mbox{$\sigma \in (0, \infty)$}.

Before stating the first result of this paper,
the following definitions 
in terms of known quantities from Gaussian process regression 
are required.
First, we define
\begin{subequations}
\label{eq:gp_mean_and_covar}
\begin{align}
  \fmu(x_{N+1}) &\doteq K^f_{N+1,1:N} \left( K^f_{1:N,1:N} + \sigma^2 K^w_{1:N,1:N} \right)^{-1} y, 
  \label{eq:gp_mean}
  \\
  \Sigma^f_\sigma(x_{N+1}) &\doteq K^f_{N+1,N+1} - K^f_{N+1,1:N} \left( K^f_{1:N,1:N} + \sigma^2 K^w_{1:N,1:N} \right)^{-1} K^f_{1:N,N+1},
  \label{eq:gp_covariance}
\end{align}
\end{subequations}
with measurements $y \doteq \begin{bmatrix}
    y_1, \ldots, y_N
\end{bmatrix}^\top \in \mathbb{R}^N$. 
For the particular
choice of noise kernel as 
\mbox{$\ks(x,x') = \sigma^2 \kw(x,x') = \sigma^2 \delta(x-x')$},
it holds that \mbox{$K^w_{1:N,1:N} = I_N$}
and
the above quantities 
respectively
correspond to 
the GP posterior mean and covariance
for 
independently and identically distributed (i.i.d.)
Gaussian measurement noise with covariance~$\sigma^2$~\cite[Chapter~2]{rasmussen_gaussian_2006}.
Additionally, 
we 
denote by
\begin{align}
  \| \gmu \|_{\Hksum}^2 &\doteq \min_{g \in \Hksum} \quad  \| g \|_{\Hksum}^2 \notag \\
  & \quad \quad \quad \quad \mathrm{s.t.} \quad g(x_i) = y_i, \> i = 1, \ldots, N, \notag \\
  &= y^\top \left( K^f_{1:N,1:N} + \sigma^2 K^w_{1:N,1:N} \right)^{-1} y
  \label{eq:minimum_norm_interpolant_sum_spaces}
\end{align}
the RKHS norm of the minimum-norm interpolant 
in the RKHS~$\Hksum$ defined for the sum of kernels $\ksum$~\cite[Theorem 58]{berlinet_reproducing_2004}.
Lastly,
\begin{align}
  \beta_\sigma^2 &\doteq \Gamma_f^2 + \frac{\Gamma_w^2}{\sigma^2} - \| \gmu \|_{\Hksum}^2
  \label{eq:beta_sigma}
\end{align}
defines the maximum norm~\eqref{eq:sup_infdim_relax_bnd_fw} in the RKHS~$\Hksum$ based on \cref{ass:RKHS_norm_fw}, 
reduced by
the minimum norm required to 
interpolate
the data.

The relaxed problem~\eqref{eq:sup_infdim_relax} admits
the following
closed-form
analytical solution. 
\begin{lemma}
  \label{thm:relaxed_upper_bound}
  Let 
  \cref{ass:RKHS_norm_fw,ass:distinct_input_locations} 
  hold.
  Then, 
  the 
  solution of 
  Problem~\eqref{eq:sup_infdim_relax} is given by
  \begin{align}
    \overline{f}^\sigma(x_{N+1}) = \fmu(x_{N+1}) + \beta_\sigma \sqrt{\Sigma^f_\sigma(x_{N+1}) }.
    \label{eq:relaxed_upper_bound}
  \end{align}
\end{lemma}

\emph{Sketch of proof:} 
First, 
following arguments from the Representer Theorem~\citep{kimeldorf_results_1971,goos_generalized_2001}, we show 
that the solution of Problem~\eqref{eq:sup_infdim_relax} is finite-dimensional. 
Next, 
two coordinate transformations are employed to reduce the number of free variables resulting from the interpolation constraint~\eqref{eq:sup_infdim_relax_data}, and to address the possible rank-deficiency of the kernel matrix $K^f_{1:N+1,1:N+1}$ for the latent function at the test and training input locations.
Finally, the problem is reduced to an equivalent linear program with a norm-ball constraint that can be analytically solved.
Expressing the solution in terms of the original coordinates then 
leads to~\eqref{eq:relaxed_upper_bound}.
The detailed proof can be found in
\cref{sec:app_relaxed_proof}.

This leads to
the relaxed bound
\mbox{$\underline{f}^\sigma(x_{N+1}) \leq f(x_{N+1}) \leq \overline{f}^\sigma(x_{N+1})$},
valid for all \mbox{$\sigma \in (0,\infty)$}. 
Due to the relaxation,
the 
obtained 
upper and lower bounds
are
conservative 
with respect to the original problems~\eqref{eq:sup_infdim_opt} and~\eqref{eq:inf_infdim_opt} 
-- 
nevertheless,
the optimal solutions 
of~\eqref{eq:sup_infdim_opt},~\eqref{eq:inf_infdim_opt} 
can be retrieved 
for a suitable choice of the noise parameter~$\sigma$, as shown in the following subsection.

\subsection{Optimal solution}
\label{sec:recovering_optimal_solution}

Our main result is formulated in the following theorem. 

\begin{theorem}
  \label{thm:optimal_solution}
  Let \cref{ass:RKHS_norm_fw,ass:distinct_input_locations} hold. 
  Then, 
  the solution of Problem~\eqref{eq:sup_infdim_opt} is given by
  \begin{align}
    \label{eq:optimal_solution_inf}
    \overline{f}(x_{N+1}) 
    &= \inf_{\sigma \in (0, \infty)} \> \overline{f}^\sigma (x_{N+1}). 
  \end{align}
\end{theorem}
\emph{Sketch of proof:} Similarly to~\cref{thm:relaxed_upper_bound}, we 
first show
that Problem~\eqref{eq:sup_infdim_opt} admits a finite-dimensional representation.
The latter is
analyzed
depending on which of the constraints~\eqref{eq:sup_infdim_opt_bnd_f} and~\eqref{eq:sup_infdim_opt_bnd_w} are active, i.e., 
influence the optimal solution
and have a corresponding strictly positive optimal Lagrange multiplier $\lambda^{f,\star}, \lambda^{w,\star}$. 
This leads to three 
non-trivial
scenarios: 
In Case 1, it holds that \mbox{$\lambda^{f,\star} > 0$}, \mbox{$\lambda^{w,\star} = 0$} and the optimal solution of~\eqref{eq:sup_infdim_opt} can be recovered by the relaxed problem for \mbox{$\sigma^\star \rightarrow \infty$},
for which 
the combined RKHS-norm constraint~\eqref{eq:sup_infdim_relax_bnd_fw} reduces to~\eqref{eq:sup_infdim_opt_bnd_f}.
In Case 2, \mbox{$\lambda^{f,\star} = 0$}, \mbox{$\lambda^{w,\star} > 0$} and \mbox{$\sigma^\star \rightarrow 0$} recovers the optimal solution,
rendering the constraint~\eqref{eq:sup_infdim_relax_bnd_fw} equivalent to~\eqref{eq:sup_infdim_opt_bnd_w}. 
In Case~3, 
both constraints are active
and
the optimal noise parameter is 
determined by
\mbox{$\sigma^\star = \sqrt{\lambda^{f,\star} / \lambda^{w,\star}} \in (0, \infty)$}, i.e.,
the
\emph{ratio of the optimal 
Lagrange 
multipliers}.
The set of active constraints at the optimal solution can be determined by case distinction, based on the feasibility of the primal solutions under the 
respective active-constraint set. 
Finally, it is shown that 
the optimal noise parameter~$\sigma^\star$ in 
all three cases 
minimizes~\eqref{eq:optimal_solution_inf}.
This is illustrated in \cref{fig:optimal_solution_illustr}, 
which depicts the optimal noise parameters~$\sigma^\star_{\sup}$,~$\sigma^\star_{\inf}$, 
for which the relaxed upper and lower bound, 
$\overline{f}^\sigma(x_{N+1})$ and $\underline{f}^\sigma(x_{N+1})$, 
correspond to the optimal bounds,
$\overline{f}(x_{N+1})$ and $\underline{f}(x_{N+1})$,
respectively.
The detailed proof can be found in \cref{sec:app_optimal_proof}.

\begin{figure}
  \includegraphics{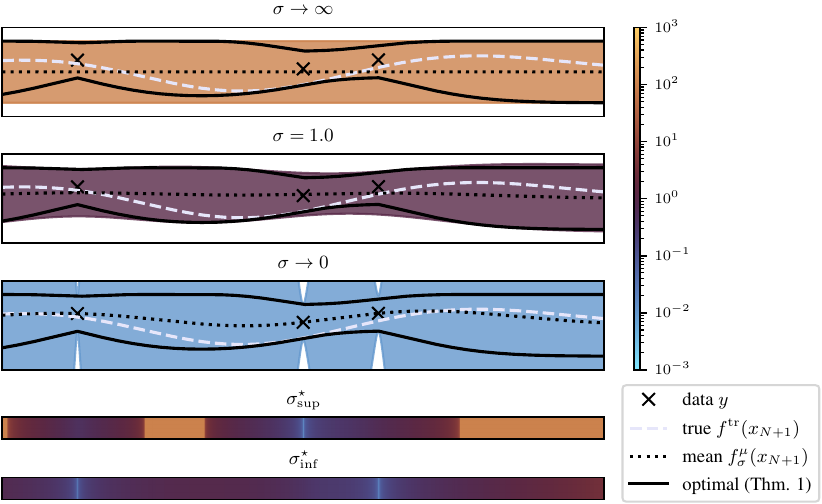}
  \caption{Illustrative example for \cref{thm:optimal_solution}. The optimal upper and lower bounds (solid black) for the (unknown) latent function $f^{\mathrm{tr}}$ (dashed white) are determined by the 
  relaxed bounds~(shaded)
  around the 
  GP
  posterior mean (dotted black) 
  for an optimal choice of noise 
  parameter~$\sigma^\star_{\sup}$ 
  (upper bound) and~$\sigma^\star_{\inf}$ (lower bound). The three upper plots show the relaxed upper and lower bounds, $\overline{f}^\sigma$ and $\underline{f}^\sigma$ for the values \mbox{$\sigma = \{ 10^{2}, 10^0, 10^{-2} \}$}, respectively. The two bottom colorbars indicate the respective optimal values~$\sigma^\star_{\sup}$ and~$\sigma^\star_{\inf}$ for the upper and lower bound. The plotted relaxed upper (lower) bounds equal the optimal upper (lower) bound for each test point where the color of the shaded area matches the color indicated in the colorbar for the optimal value $\sigma^\star_{\sup}$ ($\sigma^\star_{\inf}$).}
  \label{fig:optimal_solution_illustr}
\end{figure}

\cref{thm:optimal_solution}
reduces the solution of the infinite-dimensional optimization problem~\eqref{eq:sup_infdim_opt}
to a \emph{scalar, 
unconstrained} optimization problem 
over the noise parameter~$\sigma$.
As such, it is amenable for efficient 
iterative optimization.
Since 
running a fixed number of iterations of,
e.g., 
gradient descent
applied to Problem~\eqref{eq:optimal_solution_inf}
returns a valid, improved upper bound~$\overline{f}^\sigma(x_{N+1})$,
this allows for iterative refinement of the uncertainty envelope.
The solution thereby obtained can thus be easily integrated into existing pipelines for downstream tasks, 
such as uncertainty quantification in
streaming-data settings
or model-based reinforcement learning~\citep{deisenroth_pilco_2011,berkenkamp_safe_2017,kamthe_data-efficient_2018}.

\subsection{Special cases}\label{sec:edge_cases}

For both cases with only one active constraint,
the optimal bound 
can be determined directly in closed form, 
without optimizing for the noise parameter~$\sigma$.
In the following, we provide the respective optimal solutions, 
as well as easy-to-evaluate expressions for determining 
the active constraint set.
Noteworthy, 
the analytic solutions 
recover known bounds
in specific regression settings,
highlighting that the proposed bound 
is
a generalization thereof;
we detail these connections in
\cref{sec:discussion_noise_RKHS}. 

\paragraph{Case 1 ($\sigma \rightarrow \infty$).} 
When
the 
value~$\Gamma_w^2$ 
is sufficiently permissive,
constraint~\eqref{eq:sup_infdim_opt_bnd_w} does not influence the optimal solution of~\eqref{eq:sup_infdim_opt}.
This leads to the optimal latent function~$\fopt$ 
being chosen irrespective 
of the training data, 
while the optimal noise function~$\wopt$ 
ensures consistency with the data~\eqref{eq:sup_infdim_opt_data}.
The optimal bound~$\overline{f}(x_{N+1})$
is then
given by the prior GP covariance inflated by the full available RKHS norm $\Gamma_f$, 
recovering a classical kernel interpolation bound~\cite[Eq.~(9.7)]{fasshauer_kernel-based_2015}.

\begin{proposition}
  \label{thm:optimal_solution_case1}
  Let \cref{ass:RKHS_norm_fw,ass:distinct_input_locations} hold. 
  If 
  \begin{align}
    \left\| y - K^f_{1:N,N+1} \frac{\Gamma_f}{\sqrt{K^f_{N+1,N+1}}} \right\|_{(K^w_{1:N,1:N})^{-1}}^2 \leq \Gamma_w^2,
    \label{eq:optimal_solution_case1_feasibility}
  \end{align}
  then the solution of~\eqref{eq:sup_infdim_opt} is given as
  \begin{align}
    \overline{f}(x_{N+1}) &= \lim_{\sigma \rightarrow \infty} \> \overline{f}^\sigma (x_{N+1})
    = \sqrt{K^f_{N+1,N+1}} \Gamma_f.
    \label{eq:cost_proposition_1}
  \end{align}
\end{proposition}

The feasibility condition~\eqref{eq:optimal_solution_case1_feasibility} 
verifies
if the bound~\eqref{eq:sup_infdim_opt_bnd_w} on the noise function's RKHS norm 
allows for 
it 
to interpolate the points \mbox{$\wopt(x_i) = y_i - \fopt(x_i)$}, \mbox{$i = 1, \ldots, N$},
given the 
data-independent,
worst-case 
latent function~\mbox{$\fopt(\cdot)=\kf(\cdot,x_{N+1}) \frac{\Gamma_f}{\sqrt{K^f_{N+1,N+1}}} $}.

\paragraph{Case 2 ($\sigma \rightarrow 0$).}
For infinite-dimensional hypothesis spaces, 
a regularity constraint 
on the latent function
of the form~\eqref{eq:sup_infdim_opt_bnd_f}
is typically required to yield finite uncertainty bounds~\cite[Remark~1]{scharnhorst_robust_2023}.
Therefore, 
it is 
possible that 
merely
constraint~\eqref{eq:sup_infdim_opt_bnd_w}
is active 
only in degenerate cases
--
when
the kernel matrix~$K^f_{1:N+1,1:N+1}$ is singular, 
i.e., has rank \mbox{$r \leq N$}.
The kernel matrix can then be expressed
as \mbox{$K^f_{1:N+1,1:N+1} = \Phi_{1:N+1} \Phi_{1:N+1}^\top$}, 
where \mbox{$\Phi_{1:N+1} \in \mathbb{R}^{(N+1) \times r}$} denotes the $r$-dimensional map of linearly independent features at the training and test input locations.
This results in the following 
closed-form
optimal solution of~\eqref{eq:sup_infdim_opt}. 
\begin{proposition}
  \label{thm:optimal_solution_case2}
  Let \cref{ass:RKHS_norm_fw,ass:distinct_input_locations} hold. 
  Define \mbox{$P \doteq (\Phi_{1:N}^\top \left( K^w_{1:N,1:N} \right)^{-1} \Phi_{1:N})^{-1}$} and \mbox{$\theta^\mu \doteq P \Phi_{1:N}^\top \left( K^w_{1:N,1:N} \right)^{-1} y$}. Then, 
  if
  \begin{align}
    \left\| \theta^\star \right\|_2^2 &\doteq 
    \left\|
      \theta^\mu 
      + \frac{P \Phi_{N+1}^\top}{\| \Phi_{N+1}^\top \|_{P}} 
      \sqrt{\Gamma_w^2 
      - y^\top \left( K^w_{1:N,1:N} \right)^{-1} y 
      + \| \theta^\mu \|_{P^{-1}}^2}
    \right\|_2^2
    \leq \Gamma_f^2,
    \label{eq:optimal_solution_case2_feasibility_b}
  \end{align}
  the solution of~\eqref{eq:sup_infdim_opt} is given as
  \begin{align}
    \overline{f}(x_{N+1}) &= \lim_{\sigma \rightarrow 0} \> \overline{f}^\sigma (x_{N+1})
    = \Phi_{N+1} \theta^\mu 
    + \| \Phi_{N+1}^\top \|_{P} 
    \sqrt{
    \Gamma_w^2 - y^\top \left( K^w_{1:N,1:N} \right)^{-1} y + \| \theta^\mu \|_{P^{-1}}^2}.
    \label{eq:optimal_solution_case2_general}
  \end{align}
\end{proposition}

Since the RKHS norm of the noise function is the limiting factor in this case, 
the optimal pair of functions $(\fopt, \wopt)$
generally shows the opposite behavior as in Case 1,
utilizing the minimum RKHS norm of the noise~$\wopt$ to interpolate the data
in order to achieve a maximum value
of the latent function~$\fopt$ 
at the test point. 
The feasibility condition~\eqref{eq:optimal_solution_case2_feasibility_b} 
verifies that the RKHS norm of the optimal latent function~$\fopt$ satisfies the bound~\eqref{eq:sup_infdim_opt_bnd_f}.

Case~2 can happen in two scenarios:
For \emph{finite-dimensional} hypothesis spaces, 
i.e., 
\mbox{$\ftr(\cdot) = \begin{bmatrix}
    \phi_1(\cdot) & \ldots & \phi_r(\cdot)
\end{bmatrix} \thetatr$}
for some features~$\phi_i(\cdot)$, \mbox{$i = 1, \ldots, r$},
the latent 
function 
\mbox{$\fopt(x_{N+1})$}
generally does not have sufficient degrees of freedom to 
interpolate an arbitrary data set. 
As such, 
the 
optimal bound~$\overline{f}(x_{N+1})$ in~\eqref{eq:optimal_solution_case2_general}
consists of two components,
the value of the \emph{least-squares estimator}~\mbox{$\fmu[](x_{N+1}) \doteq \Phi_{N+1} \theta^\mu$},
as well as a term proportional to the 
maximum RKHS norm of the noise~$\Gamma_w^2$,
subtracted by
\mbox{$y^\top (K^w_{1:N,1:N})^{-1} y - \| \theta^\mu \|_{P^{-1}}$},
the minimum RKHS norm 
required to 
eliminate the offset between the least-squares estimator and the data.
For \emph{infinite-dimensional} hypothesis spaces,
the latent function~$\fopt$ can generally interpolate the offset between the optimal noise function~$\wopt$ and the training data;
however, 
neglecting the RKHS-norm constraint~\eqref{eq:sup_infdim_opt_bnd_f} on~$\fopt$ only leads to sensible estimates
when the test point coincides with a training input location.
In this case, the feature vector \mbox{$\Phi_{N+1} \in \mathbb{R}^{(N+1) \times r}$} has rank~$N$,
which simplifies the general result in~\cref{thm:optimal_solution_case2}.

\begin{corollary}
  \label{thm:optimal_solution_case2_corollary_testtrain}
  Let \cref{ass:RKHS_norm_fw,ass:distinct_input_locations} hold. 
  Suppose that $K^f_{1:N,1:N}$ 
  is 
  invertible and \mbox{$x_{N+1} = x_k \in \mathbb{X}$}.
  Then, if
  \begin{align}
    \left\| \theta^\star \right\|_2^2 
    &\doteq
    \left\|
      y + 
      K^w_{1:N,k} 
      \frac{\Gamma_w}{\sqrt{K^w_{k,k}}}           
    \right\|_{(K^f_{1:N,1:N})^{-1}}^2
    \leq \Gamma_f^2,
    \label{eq:optimal_solution_case2_feasibility_a}
  \end{align}
  the solution of~\eqref{eq:sup_infdim_opt} is given as
  \begin{align}
    \overline{f}(x_{N+1}) &= \lim_{\sigma \rightarrow 0} \> \overline{f}^\sigma (x_{N+1})
    = y_k + \sqrt{K^w_{k,k}} \Gamma_w.
    \label{eq:constraint_proposition1}
  \end{align}
\end{corollary}
In this case, 
the optimal solution at a training input location~\mbox{$x_{N+1} = x_k \in \mathbb{X}$} is given by the corresponding measurement~$y_i$,
inflated by the maximum RKHS norm~$\Gamma_w$ of the noise function.

\section{Related work and discussion}

In this section, we discuss 
the obtained bounds 
also in view of known results 
from the literature.
In particular, 
in \cref{sec:discussion_noise_RKHS} we detail known bounds that are recovered as special cases of \cref{thm:optimal_solution}.
In \cref{sec:discussion_deterministic,sec:discussion_stochastic},
we respectively compare~\cref{thm:optimal_solution} with 
\emph{deterministic} bounds,
obtained for 
bounded noise sequences,
and \emph{probabilistic} bounds, for noise sequences with an assumed probability distribution. Specifically, \cref{sec:discussion_stochastic} 
encompasses a thorough numerical comparison of the bounds,
including an application on safe control 
that shows the effectiveness of the proposed bounds on a downstream task.

\subsection{Recovering existing bounds as particular cases}\label{sec:discussion_noise_RKHS}

\paragraph*{Linear regression under energy-bounded noise} We first elucidate the connection between Assumption~\ref{ass:RKHS_norm_fw} 
on the deterministic noise \emph{function}
and energy-boundedness of the noise \emph{sequence}. 
As 
a straightforward consequence of the Representer Theorem~\citep{wahba_spline_1990,goos_generalized_2001},
for the presented results in this paper,
the values of the noise function outside the training input locations are irrelevant~(see \cref{sec:app_finite_dimensional_solution}):
there exists a noise-generating function~$\wtr$ for the data set~$\mathcal{D}$ 
satisfying \cref{ass:RKHS_norm_fw} 
if and only if
the
minimum-norm interpolant~$\wmu[]$ 
of the (unknown) noise realizations
$w_{\mathbb{X}} \doteq \begin{bmatrix}
  w(x_1), \ldots, w(x_N)
\end{bmatrix}^\top$
satisfies
the 
RKHS-norm 
bound
\mbox{$\| \wmu[] \|_{\Hkw}^2 < \Gamma_w^2$},
where
  \begin{align*}
    \| \wmu[] \|_{\Hkw}^2 &\doteq \min_{w \in \Hkw} \quad  \| w \|_{\Hkw}^2 \\
    & \quad \quad \quad \mathrm{s.t.} \quad w(x_i) = \wtr(x_i), \> i = 1, \ldots, N, \\
    &= w_{\mathbb{X}}^\top \left( K^w_{1:N,1:N} \right)^{-1} w_{\mathbb{X}}.
  \end{align*}
Therefore, 
instead of imposing a maximum RKHS norm on~$\wtr$,
one could equivalently 
assume a bounded RKHS norm for the minimum-norm interpolant generating the data set. For the Dirac noise kernel \mbox{$\kw(x,x') = \delta(x - x')$},
since \mbox{$K^w_{1:N,1:N} = I_N$},
the bounded-RKHS-norm assumption on the noise function 
implies 
bounded energy
of
the noise sequence,
i.e.,
\begin{equation*}
    \| \wmu[] \|_{\Hkw}^2 = \sum_{i=1}^N \wtr(x_i)^2 < \Gamma_w^2.
\end{equation*}
The assumption of bounded energy for the data set has been employed by~\cite{fogel_system_1979}
to obtain bounds for the latent function in the setting of linear regression,
i.e., finite-dimensional hypothesis spaces. Using the notation adopted in \cref{sec:edge_cases}, the non-falsified 
parameter set
is obtained as
\begin{equation}
  \| \thetatr - \theta^\mu \|_{P^{-1}}^2 
  \leq \Gamma_w^2 - y^\top y + \| \theta^\mu \|_{P^{-1}}^2,
  \label{eq:fogel_bound}
\end{equation}
see~\cite[Eq.~(3)]{fogel_system_1979}.
\cref{thm:optimal_solution_case2} shows that the obtained bound recovers this known 
result from
set-membership estimation
for finite-dimensional hypothesis spaces. 
In fact, 
for the Dirac noise kernel,
the worst-case realization of the unknown parameters is given by~$\thetaopt$ in~\eqref{eq:optimal_solution_case2_feasibility_b}, 
for which~\eqref{eq:fogel_bound} holds with equality. 
Note that the 
optimal bound in \cref{thm:optimal_solution}
does
not only recover the bounds by~\cite{fogel_system_1979} for the linear-regression case, but, moreover, 
provides bounds under the additional complexity constraint~\mbox{$\| \thetatr \|_2^2 \leq \Gamma_f^2$}.

\paragraph*{Noise-free kernel interpolation}
Building upon the kernel interpolation bound by \cite{weinberger_optimal_1959} (see~\cite[p.~192]{wendland_scattered_2004},~\cite[Section~9.3]{fasshauer_kernel-based_2015}), under Assumption~\ref{ass:RKHS_norm_fw} the following bound can be derived, cf.~\cite[Proposition~1]{maddalena_deterministic_2021}:
\begin{equation}
    \left| f(x_{N+1}) - \fmu[0](x_{N+1}) \right| \leq \sqrt{\Gamma_f^2 - \| \fmu[0] \|_{\Hkf}^2 } \sqrt{\Sigma^f_0(x_{N+1})}, \label{eq:pseudo_golomb}
\end{equation}
where $\fmu[0]$ is the minimum-norm interpolant in the RKHS $\Hkf$ (see~\eqref{eq:minimum_norm_interpolant_sum_spaces}) and $\sqrt{\Sigma^f_0(x_{N+1})}$ is commonly referred to as the \emph{power function}. The relaxed bound in \cref{thm:relaxed_upper_bound} 
generalizes this result:
for noise-free measurements, i.e., \mbox{$\Gamma_w^2 \rightarrow 0$}, and in the limit for the noise parameter \mbox{$\sigma \rightarrow 0$},~\eqref{eq:pseudo_golomb} is recovered exactly by noting that the bound in~\cref{thm:relaxed_upper_bound} is symmetric around the estimate.

\subsection{Comparison with existing deterministic bounds}
\label{sec:discussion_deterministic}

\paragraph{Interpolation using sum-of-kernels} For the Dirac noise kernel~\mbox{$\kw(x,x') = \delta(x - x')$}, 
uncertainty
bounds have also been obtained by~\cite[Section~6.4]{kanagawa_gaussian_2025}
based on 
the minimum-norm interpolant \mbox{$\gmu \in \Hksum$} 
using the
sum of kernels~$\ksum$. 
Utilizing the fact that,
for all \mbox{$x \notin \mathbb{X}$},
the GP posterior mean~\eqref{eq:gp_mean} and covariance~\eqref{eq:gp_covariance}
are equal to the interpolant~$\gmu$ and the corresponding power function, respectively,
a
bound
on the true data-generating function \mbox{$g \in \Hksum$} 
has been established by~\cite[Corollary~6.8]{kanagawa_gaussian_2025}.
However, 
the bound does not
take into account
the actual value of the measurements \mbox{$\{ g(x_i) = y_i \}_{i=1}^N$},
but rather the worst-case realization thereof,
rendering it conservative.
Additionally,
the bound is only valid for the data-generating process~$g$
and does not 
provide
bounds 
for
the latent function.

\paragraph{Point-wise bounded noise} As energy-boundedness is a weaker assumption than point-wise boundedness of the noise, \cref{thm:optimal_solution} can also be applied in the setting of point-wise bounded noise,
see \cref{sec:discussion_noise_RKHS}.
In this setting, \citep{maddalena_deterministic_2021,reed_error_2025} provide closed-form, yet conservative, bounds for the latent function under an RKHS-norm constraint on the latter. 
The bounds are improved upon by~\cite{scharnhorst_robust_2023}, 
which provides optimal point-wise bounds for the latent function. 
As the bounded-energy and pointwise-boundedness assumptions are equivalent for \mbox{$N = 1$} data points, so are the bounds by~\cite{scharnhorst_robust_2023} and \cref{thm:optimal_solution} in this case.
For larger data sets, 
under the point-wise-boundedness assumption,
the optimal bounds by~\cite{scharnhorst_robust_2023} are tighter than 
the optimal bounds in~\cref{thm:optimal_solution} 
obtained under the weaker bounded-energy assumption. 
Still,
their
computation 
relies on solving a constrained convex program, 
cf. \cite[Eq.~(6)]{scharnhorst_robust_2023},
whose number of optimization variables is proportional to the 
number of training data points~$N$.
Additionally. this 
optimization problem 
has to be solved to optimality in order to obtain valid bounds for the latent function, while optimization over the noise parameter~$\sigma$ in~\eqref{eq:optimal_solution_inf} returns a valid upper bound for all \mbox{$\sigma \in (0, \infty)$}.

\subsection{Comparison with existing probabilistic bounds}
\label{sec:discussion_stochastic}

High-probability bounds for Gaussian-process regression are derived in~\citep{srinivas_information-theoretic_2012,abbasi-yadkori_online_2013,fiedler_practical_2021,molodchyk_towards_2025}, which are generally of the form 
\begin{equation}
    \mathrm{Pr}\left[|f(x)-f^\mu_\sigma(x)|\leq \beta_{\sigma}\sqrt{\Sigma_\sigma^f(x)} \> \forall x\in\mathcal{X}\right]\geq p. \notag
\end{equation}
Compared to Theorem~\ref{thm:optimal_solution}, these bounds hold with a user-chosen probability \mbox{$p\in(0,1)$}, use a fixed constant \mbox{$\sigma>0$}, but otherwise have the same structure and the same assumption on the latent function~$\ftr$ in terms of a known bound on its RKHS norm.  
However, while the proposed analysis considers energy-bounded noise, \mbox{$\| \wtr \|_{\Hkw}^2 < \Gamma_w^2$}  (Assumption~\ref{ass:RKHS_norm_fw}), these results apply to (conditionally) independent sub-Gaussian noise~\citep{srinivas_information-theoretic_2012,abbasi-yadkori_online_2013,fiedler_practical_2021}.
This is a stronger\footnote{%
To be precise, if the noise is sub-Gaussian, we can derive an energy bound $\Gamma_w^2$, such that $w^{\mathrm{tr}}$ satisfies Assumption~\ref{ass:RKHS_norm_fw} with the kernel \mbox{$\kw(x,x')=\delta(x-x')$} and a desired probability \mbox{$p\in(0,1)$}. However, the converse is not true as energy-bounded noise may be correlated and biased.}
requirement as it does not allow for biased or correlated noise, which can be difficult to ensure in real-world experiments. 
Nonetheless, both the proposed bound and existing high-probability bounds can be applied in case of independent, zero-mean and bounded noise;
in the following, we numerically investigate the conservativeness of the bounds in this setting%
\footnote{The code to reproduce the experiments is publicly available at \url{https://gitlab.ethz.ch/ics/bounded-energy-rkhs-bounds} and at \url{https://doi.org/10.3929/ethz-c-000785083}.}%
.
\paragraph{Numerical comparison}
In this experiment, the size of the uncertainty regions is compared. 
Using a squared-exponential kernel 
\mbox{$\kf(x,x') = \exp(-\|x - x'\|^2 / \ell^2)$}, 
\mbox{$\ell = 1$},
for the latent function, as well as a Dirac noise kernel 
\mbox{$\kw(x,x') = \delta(x - x')$} on the domain 
\mbox{$\mathcal{X} = [0, 4]$}, 
random latent functions are generated
with \mbox{$\| \ftr \|_{\Hkf}^2 = \Gamma_f^2 = 1$}. 
Training data are 
sampled
based on measurement noise following a zero-mean truncated Gaussian distribution with standard deviation and bounded absolute value equal to \mbox{$\epsilon = 0.01$}, 
which is $R$-sub-Gaussian for \mbox{$R = \epsilon$}.
The corresponding noise-energy bound is derived as \mbox{$\Gamma_w^2 = N \epsilon^2$}.
We compare the 
proposed
bound (Theorem~\ref{thm:optimal_solution}), which is optimal given only the information \mbox{$\| w^{\mathrm{tr}} \|_{\Hkw}^2 \leq \Gamma_w^2$}, the relaxed bound (Lemma~\ref{thm:relaxed_upper_bound}) with \mbox{$\sigma=\epsilon$}, and a standard high-probability error bound%
~\citep{abbasi-yadkori_online_2013}, cf. \citep[Eq.~(7)]{fiedler_safety_2024}, 
which uses only sub-Gaussianity of the noise and provides a valid bound 
with probability \mbox{$p = 0.99$}, 
similar to~\citep{srinivas_information-theoretic_2012,fiedler_practical_2021}.
Optimality of the 
numerical solution to~\eqref{eq:optimal_solution_inf}
is guaranteed 
by solving a convex reformulation of~\eqref{eq:sup_infdim_opt} (see~\cref{sec:app_finite_dimensional_solution}) using~\texttt{CVXPY}~\citep{diamond_cvxpy_2016,agrawal_rewriting_2018}.
Figure~\ref{fig:stoch_compare} compares the area of the uncertainty region
for 
\mbox{$N=1,\dots, 10^3$}
randomly sampled
training
points,
averaged over $10^3$~runs
for
randomly sampled
latent
functions~$f^{\mathrm{tr}}$. 
In the low-data regime, the proposed 
optimal and relaxed bounds,
leveraging energy-boundedness of the noise, 
are significantly less conservative. 
However, as multiple similar data points may provide no additional information without probabilistic information, 
they
do not significantly improve after a certain number of data points $N$. In contrast, the probabilistic bound leverages independence, asymptotically attaining smaller uncertainty bounds with increasing data. 
However, it should be noted that these probabilistic bounds are only valid if indeed the noise is (conditionally) independent and zero-mean; otherwise, these shrinking confidence 
intervals
would be misleading.

\begin{figure}
  \includegraphics{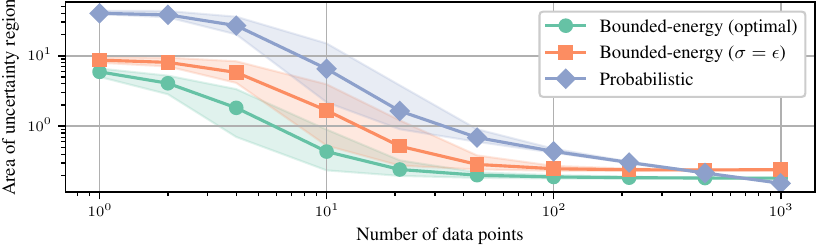}
  \caption{Numerical comparison of area of 
  uncertainty 
  region 
  for increasing number of data points~$N$ with 
  \mbox{$\{5\%, 95\%\}$}-percentiles 
  shown in shade.
  }
  \label{fig:stoch_compare}
\end{figure}

\paragraph{Safe control for uncertain nonlinear systems}

Lastly, we demonstrate the application of the proposed bounds to the downstream task of safe control. 
Consider the uncertain (for simplicity scalar) nonlinear dynamical system
\begin{align}
  \label{eq:CBF_dynamics}
  x(k + 1) &= 
  f^{\mathrm{known}}(x(k), u(k)) + \ftr(x(k), u(k)), %
\end{align}
with 
known dynamics $f^{\mathrm{known}}$ and unknown residual dynamics $f^{\mathrm{tr}}$.
Given 
the current state 
$x(k) \in \mathbb{R}^{}$
of the system
at time $k \in \mathbb{N}$, 
the goal is to 
find 
an optimal
control input
\mbox{$u(k) \in \mathbb{R}^{}$}
that
minimizes a user-defined cost function~$c(x(k), u(k))$,
subject to
a safety-critical constraint,
\mbox{$f^{\mathrm{known}}(x(k), u(k)) + \ftr(x(k), u(k)) \geq (1 - \gamma) x(k)$} --
similar to a control barrier function%
~\citep{agrawal_discrete_2017,ames2019control,jagtap_control_2020-1}.
The 
uncertainty in 
the 
system dynamics
is 
handled
by
leveraging the proposed (the probabilistic) bound 
to
enforce 
\emph{robust}
constraint satisfaction
for all functions in the uncertainty set,
containing the unknown function~$\ftr$ (with probability $p$).
Importantly, 
this makes
tight uncertainty bounds 
desirable,
since they generally lead to 
lower costs and a larger feasible region,
where safety of the control input can be guaranteed.

We compare the 
bounds
for the following example setup:
$f^{\mathrm{known}}(x,u) \doteq 0.5 x + u - 1$, 
$\ftr(x, u) \doteq \exp(-x^2) \sin(10 x)$, 
$c(x,u) \doteq (f^{\mathrm{known}}(x,u) + \fmu(x,u))^2 + u^2$, $\kf$, $\kw$ as above with \mbox{$\ell = \sqrt{2}/20$}.
For the proposed bound, the 
optimization problem is formulated using the relaxed bound in~\cref{eq:relaxed_upper_bound}, optimizing $\sigma \in (0, \infty)$ and $u(k) \in [-2,2]$ \emph{simultaneously}; for the probabilistic bounds, $\sigma = \epsilon$ is fixed with the same noise assumptions as in~\cref{sec:discussion_stochastic}; see~\cref{sec:appendix_CBF_example} for implementation details.
Additionally, we implement the proposed bound using only the nearest 10~training points to construct the uncertainty bounds.
\cref{fig:cbf} shows the success rate in terms of the share of feasible problems 
for an increasing amount of training data on a grid of $500$ test points in the domain $x(k) \in [-2,2]$, repeated $20$~times with random noise realizations.  Due to the smaller uncertainty bounds, the proposed bound using the full data set achieves the highest success rate,
albeit at a high computational cost. In contrast, using the probabilistic bounds and the subset-of-data variant of the proposed bounds leads to similar success rates, with the latter exhibiting significantly lower computation times due to its independence of the number of training points. 
Note that utilizing the probabilistic bounds with test-point-dependent subsets of data would generally deteriorate the probability~$p$ of their joint validity for all test points, 
compromising the controller's safety guarantees.

\begin{figure}
  \includegraphics{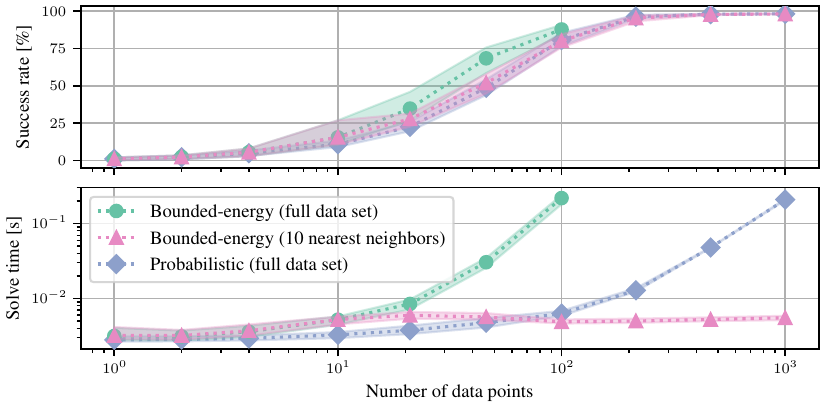}
  \caption{Application of uncertainty bounds for safe control. Success rate (upper plot) and solve time (lower plot) with $\{ 5\%, 95\% \}$-percentiles shown in shade.
  }
  \label{fig:cbf}
\end{figure}

\subsection{Limitations}
\label{sec:limitations}

The obtained bounds suffer from common criticalities and limitations of kernel-based learning, which are (a) dealing with kernel mis-specification and (b) knowing valid RKHS-norm bounds.
Both issues are typically addressed empirically, (a) by hyper-parameter tuning via cross-validation or maximum-likelihood estimation (\citep{wahba_spline_1990}, \citep[Section 5.4]{rasmussen_gaussian_2006}, \citep{karvonen_maximum_2020-3}) and (b) by estimating the bound value from data~\citep{csaji_nonparametric_2022,tokmak_pacsbo_2024}.
While a rigorous investigation of~(a) is beyond the scope of this paper, 
we note that mis-specification of~$\kf$ may also be compensated by an inflated RKHS-norm bound~$\Gamma_w$.
Regarding~(b), 
for the latent function,
this 
is a 
common assumption in the literature on kernel-based uncertainty bounds;
for the noise function, 
we discuss in \cref{sec:discussion_noise_RKHS} how it generalizes the common setting of energy-bounded noise. 
Under-estimation in the RKHS-norm bounds could be detected by checking feasibility of the optimization problem. 
Conversely, we point out that considering conservative values of $\Gamma_f$ and $\Gamma_w$ 
merely results in
a sublinear inflation of the computed uncertainty envelope, cf.~\cref{eq:beta_sigma}.
Nevertheless, further research will be devoted to rigorously assessing the robustness of the obtained bounds with respect to possible mis-specifications in (a) and (b).

\section{Conclusions}
\label{sec:conclusions}
The main contribution of this paper is an optimization-based, \emph{distribution-free} bound for kernel-based estimates that is tight, even in the non-asymptotic, low-data
regime, and that can handle 
correlated and biased
noise sequences. 
The proposed bound generalizes known results from kernel interpolation in the noise-free setting and from linear regression under energy-bounded noise.
In the case of bounded sub-Gaussian noise, the numerical results highlight the competitiveness of the bound with existing probabilistic bounds in terms of its conservatism,
and showcase its high potential for safe control as a downstream task.
Moreover, the 
experiments highlight how the
deterministic nature of the proposed bound enables the rigorous certification of subset-of-data selection or mixture-of-experts strategies to handle large data sets.
Future work may 
investigate the effectiveness of the proposed bound for further downstream tasks, such as Bayesian optimization or model-based reinforcement learning.

\begin{ack}
  This work was supported by the European Union’s Horizon 2020 research
and innovation programme, Marie Skłodowska-Curie grant agreement No.
953348, ELO-X. The authors thank the anonymous reviewers for their constructive comments. AL thanks Philipp Hennig, Motonobu Kanagawa and Manish Prajapat for helpful discussions.
\end{ack}

\bibliographystyle{abbrvnat}
\bibliography{biblio_final}

\newpage

\appendix

{\Large \bfseries Technical Appendix}
\vspace{0.8em}

The following sections contain the proofs of the mathematical claims made in the paper, as well as implementation details for the numerical examples. Specifically,~\cref{sec:app_finite_dimensional_solution} collects ancillary results, 
showing that the original infinite-dimensional problems yielding the upper- and lower bounds admit a finite-dimensional representation, which is the first step in computing their analytical solutions; additionally, it also presents two coordinate transformations that are useful for the following results. \cref{sec:app_relaxed_proof} provides the proof of~\cref{thm:relaxed_upper_bound}. \cref{sec:app_optimal_proof} contains the proof of~\cref{thm:optimal_solution}, together with those for the special cases presented in~\cref{thm:optimal_solution_case1,thm:optimal_solution_case2,thm:optimal_solution_case2_corollary_testtrain}.  
Finally, \cref{sec:appendix_CBF_example} provides further implementation details for the numerical example on ``Safe control for uncertain nonlinear systems'' in \cref{sec:discussion_stochastic}.

\allowdisplaybreaks
\numberwithin{equation}{section}
\numberwithin{lemma}{section}

\newpage

\section{Finite-dimensional representation of optimization problems}
\label{sec:app_finite_dimensional_solution}

In this Section we first prove that optimization problems~\eqref{eq:sup_infdim_opt} and~\eqref{eq:sup_infdim_relax} admit a finite-dimensional representation  (\cref{thm:finite_dimensional_solution} and~\cref{thm:finite_dimensional_solution_relax} in~\cref{sec:representer}). Next, in~\cref{sec:app_coordinate_trafo} we present two coordinate transformations that will be deployed in the remaining sections. 

\subsection{Representer Theorems}\label{sec:representer}

By using standard ideas from the representer theorem~\cite{kimeldorf_results_1971,goos_generalized_2001} and \cite[Appendix C.1]{scharnhorst_robust_2023}, 
it can be established that the maximizer of~\eqref{eq:sup_infdim_opt} is finite-dimensional. 

\begin{lemma}
\label{thm:finite_dimensional_solution}
      A global maximizer of Problem~\eqref{eq:sup_infdim_opt} is given by
  \begin{align}
    \label{eq:maximizer_representer}
    \fopt(\cdot) = \sum_{i=1}^{N+1} \kf(\cdot, x_i) \alpha^f_{i} \in \Hkf, \qquad
    \wopt(\cdot) = \sum_{i=1}^{N} \kw(\cdot, x_i) \alpha^w_{i} \in \Hkw.   
  \end{align}
  Furthermore, Problem~\eqref{eq:sup_infdim_opt} is equivalent to the following finite-dimensional problem with~\mbox{$c^w \doteq K^w_{1:N,1:N} \alpha^w \in \mathbb{R}^N$}:
  \begin{subequations}
  \label{eq:sup_findim_opt_app_alpha}
  \begin{align}
    \overline{f}(x_{N+1}) = \sup_{\substack{
      \alpha^f \in \mathbb{R}^{N+1}, \\
      c^w \in \mathbb{R}^{N}
    }} 
    \quad & K^f_{N+1,1:N+1} \alpha^f \\
    \mathrm{s.t.} \quad 
    & K^f_{1:N,1:N+1} \alpha^f + c^w = y, \\
    & (\alpha^f)^\top K^f_{1:N+1,1:N+1} \alpha^f - \Gamma_f^2 \leq 0, \\
    & \frac{1}{\sigma^2} \left( (c^w)^\top (K^{w}_{1:N,1:N})^{-1} c^w - \Gamma_w^2 \right) \leq 0.
  \end{align}
\end{subequations}
\end{lemma}
\begin{proof}
    Let 
  \mbox{$\mathbb{X} = \{ x_1, \ldots, x_N \}$}
  be the set of training input locations and
  \mbox{$\mathbb{X}_+ = \mathbb{X} \cup \{ x_{N+1} \}$},
  the same set augmented with the test point.
  We %
  denote by \mbox{$\Hkf^{\parallel} = \{ f \in \Hkf: f \in \mathrm{span}(\kf(\cdot, x_i), x_i \in \mathbb{X}_+) \}$}
  the span of kernel functions evaluated at the training and test input locations,
  as well as 
  by 
  $\Hkf^\perp$
  its orthogonal complement, i.e.,
  \mbox{$\Hkf^\perp = \{ f^\perp \in \Hkf: \langle f^\perp, f^\parallel \rangle_{\Hkf} = 0 \text{ for all } f^\parallel \in \Hkf^\parallel \}$}.
  Hence, any function \mbox{$f \in \Hkf$}
  can be written as \mbox{$f = f^\parallel + f^\perp$}, where \mbox{$f^\parallel \in \Hkf^\parallel$} and \mbox{$f^\perp \in \Hkf^\perp$}.
  Note that 
  the cost of the optimization problem is
  \mbox{$f(x_{N+1}) = \langle f, k(x_{N+1}, \cdot) \rangle_{\Hkf}$}, 
  which 
  is 
  insensitive to 
  the 
  orthogonal 
  part $f^\perp$.
  Regarding the constraints, note that all functions 
  \mbox{$f^\perp \in \Hkf^\perp$} 
  do not affect the equality constraint~\eqref{eq:sup_infdim_opt_data} while tightening the inequality constraints~\eqref{eq:sup_infdim_opt_bnd_f}; hence, it is optimal to set \mbox{$f^\perp \equiv 0$}. 
  By the same arguments, 
  it is optimal to set \mbox{$w^\perp \equiv 0$}.
  where
  the orthogonal complement 
  is defined with respect to the finite-dimensional subspace
  \mbox{$\Hkw^{\parallel} = \{ w \in \Hkw: w \in \mathrm{span}(\kw(\cdot, x_i), x_i \in \mathbb{X}) \}$},
  which excludes $\kw(\cdot, x_{N+1})$
  as the cost is insensitive to $w(x_{N+1})$, the value of the noise function at the test point.
  Hence, it follows that for all functions \mbox{$f \in \Hkf$} and \mbox{$w \in \Hkw$},
  the respective orthogonal parts \mbox{$f^\perp \in \Hkf^{\perp}$} and \mbox{$w^\perp \in \Hkw^{\perp}$} can be set to zero without affecting feasibility or optimality of the candidate function.

  Next, we show that the supremum is actually attained, i.e., that the optimizers $\fopt$ and $\wopt$ 
  are elements of the respective
  finite-dimensional subspaces $\Hkf^{\parallel}$ and $\Hkw^{\parallel}$.
  First, we note that the norm constraints~\eqref{eq:sup_infdim_opt_bnd_f} and~\eqref{eq:sup_infdim_opt_bnd_w} define closed and bounded sets
  in the metric spaces $\Hkf$ and $\Hkw$, respectively.
  By the Cauchy-Schwartz inequality, the norm constraints~\eqref{eq:sup_infdim_opt_bnd_f} and~\eqref{eq:sup_infdim_opt_bnd_w} 
  imply bounds on the pointwise evaluation of $f$ and $w$:
  \mbox{$\left| f(x_i) \right| = \langle f, \kf(x_i, \cdot) \rangle_{\Hkf} \leq \| \kf(x_i,\cdot) \|_{\Hkf} \| f \|_{\Hkf} \leq c_f \Gamma_f$},
  where \mbox{$c_f \doteq \sup_{x \in \mathcal{X}} \sqrt{\kf(x,x)}$}. Similarly, it holds that 
  \mbox{$\left| w(x_i) \right| = \langle w, \kw(x_i, \cdot) \rangle_{\Hkw} \leq \| \kw(x_i,\cdot) \|_{\Hkw} \| w \|_{\Hkw} \leq c_w \Gamma_w$},
  where \mbox{$c_w \doteq \sup_{x \in \mathcal{X}} \sqrt{\kw(x,x)}$}. 
  Note that $c_w,c_f<\infty$ holds by Assumption~\ref{ass:RKHS_norm_fw}. 
  Jointly with the data interpolation constraint~\eqref{eq:sup_infdim_opt_data}, 
  this defines closed and bounded sets 
  \begin{align*}
    D_i \doteq \left\{ (f(x_i), w(x_i)) 
  \> \middle| \> f \in \Hkf^\parallel, w \in \Hkw^\parallel, \: f(x_i) + w(x_i) = y_i, \left| f(x_i) \right| \leq c_f \Gamma_f, \left| w(x_i) \right| \leq c_w \Gamma_w \right\}
  \end{align*}
  in~$\mathbb{R}^{2}$,
  for all \mbox{$i = 1, \ldots, N$}.
  As the evaluation functionals \mbox{$\mathbb{E}^f_{x_i}(f) = f(x_i) = \langle f, \kf(x_i, \cdot) \rangle_{\Hkf}$}
  and \mbox{$\mathbb{E}^w_{x_i}(w) = w(x_i) = \langle f, \kw(x_i, \cdot) \rangle_{\Hkw}$} corresponding to the RKHSs $\Hkf$ and $\Hkw$, respectively, are linear and continuous, the pre-image of $D_i$, \mbox{$\mathrm{pre}(D_i) = \left\{ (f,w) \> \middle| \>  (f(x_i), w(x_i)) \in D_i \right\}$}, is closed  in \mbox{$\Hkf^\parallel \times \Hkw^\parallel$}, for all \mbox{$i = 1, \ldots, N$}. 
  Furthermore, the intersection of $\mathrm{pre}(D_i)$, $i=1,\dots,N$ and the bounded norm constraints~\eqref{eq:sup_infdim_opt_bnd_f} and~\eqref{eq:sup_infdim_opt_bnd_w} is closed and also bounded in  $\Hkf^\parallel\times\Hkw^\parallel$
  i.e., the feasible set of Problem~\eqref{eq:sup_infdim_opt} is closed and bounded.
  Since $\Hkf^\parallel$ and $\Hkw^\parallel$ are finite-dimensional, 
  by the Heine-Borel theorem, the feasible set is compact;
  the value of the continuous objective~$f(x_{N+1}) = \langle f, \kf(x_{N+1}, \cdot) \rangle_{\Hkf}$
  is thus attained by the Weierstrass extreme value theorem.

  The finite-dimensional formulation~\eqref{eq:sup_findim_opt_app_alpha_relax} follows directly from inserting the finite-dimensional representations of~$\fopt \in \Hkf^\parallel$ and~$\wopt \in \Hkw^\parallel$ in~\eqref{eq:sup_infdim_opt} and defining $c^w = K^w_{1:N,1:N} \alpha^w$.  
\end{proof}

Similarly, we now prove that the relaxed infinite-dimensional problem~\eqref{eq:sup_infdim_relax} admits a finite-dimensional representation.

\begin{lemma}
  \label{thm:finite_dimensional_solution_relax}
    A global maximizer of Problem~\eqref{eq:sup_infdim_relax} is given by~\eqref{eq:maximizer_representer}, and the resulting finite-dimensional problem can be written as
\begin{subequations}
  \label{eq:sup_findim_opt_app_alpha_relax}
  \begin{align}
    \overline{f}^\sigma(x_{N+1}) = \sup_{\substack{
      \alpha^f \in \mathbb{R}^{N+1}, \\
      c^w \in \mathbb{R}^{N}
    }} 
    \quad & K^f_{N+1,1:N+1} \alpha^f \\
    \mathrm{s.t.} \quad 
    & K^f_{1:N,1:N+1} \alpha^f + c^w = y, \label{eq:sup_findim_opt_app_alpha_relax_data} \\
    & \| \alpha^f \|_{K^f_{1:N+1,1:N+1}}^2 + \| c^w \|_{(\sigma^2 K^{w}_{1:N,1:N})^{-1}}^2 \leq \Gamma_f^2 + \frac{\Gamma_w^2}{\sigma^2}. \label{eq:sup_findim_opt_app_alpha_relax_bnd_fw}
  \end{align}
\end{subequations}
\end{lemma}
\begin{proof}
  Analogous to \cref{thm:finite_dimensional_solution}, it holds that 
  setting
  $f^\perp \equiv 0$ 
  retains optimality of any candidate function $f \in \Hk$, with $f = f^\parallel + f^\perp$.
 Similarly for the noise, it holds that $w^\perp \equiv 0$.
  Attainment of the supremum is also established along the lines of \cref{thm:finite_dimensional_solution},
  noting that the sum of norm-constraints~\eqref{eq:sup_infdim_relax_bnd_fw} defines a closed and bounded set in $\Hkf^\parallel \times \Hkw^\parallel$.
  Finally, the finite-dimensional optimization problem follows from replacing $f,w$ with their finite-dimensional expressions~\eqref{eq:maximizer_representer}.
\end{proof}

\subsection{Coordinate transformations}
\label{sec:app_coordinate_trafo}

We now present two transformations that will allow us to simplify the finite-dimensional representations~\eqref{eq:sup_findim_opt_app_alpha} and~\eqref{eq:sup_findim_opt_app_alpha_relax}.
The first one will be used to deal with the possible rank-deficiency of the kernel matrix, the second one, to decompose the hypothesis space into orthogonal features. A subset of the corresponding weights will be fully determined by the training data, while the remaining ones will be adversarially chosen to obtain the worst-case value of the latent function at the test point. We point the interested reader to \cite{muller_newton_2009,pazouki_bases_2011} for details on similar basis transformations for kernel spaces.

\subsubsection*{Eliminating the null space of the kernel matrix}

For degenerate kernel functions, i.e., finite-dimensional hypothesis spaces, 
as well as in the case when the test point coincides with a training data point, 
the kernel matrix~$K^f_{1:N+1,1:N+1}$ associated with the latent function~$f$ can be singular. 
To handle the rank-deficiency, let us denote the rank of the matrix $K^f_{1:N+1,1:N+1}$ by $r$, which satisfies $r\leq N+1$ by definition. 
To eliminate redundant variables, we employ a singular value decomposition (SVD) of~$K^f_{1:N+1,1:N+1}$:
\begin{align}
  K^f_{1:N+1,1:N+1} &=
  \begin{bmatrix}
    K^f_{1:N,1:N} & K^f_{1:N,N+1} \\
    K^f_{N+1,1:N} & K^f_{N+1,N+1}
  \end{bmatrix}
  = V S V^\top \\
  &= 
  \begin{bmatrix}
    V_{11} & V_{12} \\
    V_{21} & V_{22}
  \end{bmatrix}
  \begin{bmatrix}
    S_r & 0 \\
    0 & 0
  \end{bmatrix}
  \begin{bmatrix}
    V_{11}^\top & V_{21}^\top \\
    V_{12}^\top & V_{22}^\top
  \end{bmatrix} %
  \label{eq:svd_kf}
\end{align}
Thereby we have partitioned the rows of the
orthonormal matrix $V$, 
with \mbox{$V V^\top = I$}, 
according to the separation of 
$K^f_{1:N+1,1:N+1}$
into evaluations at the training points and the test point. 
Note that if \mbox{$K^f_{1:N+1,1:N+1}$} has full rank, i.e., \mbox{$r = N+1$}, then 
the diagonal and
positive-definite matrix $S_r$, containing the non-zero singular values of $K^f_{1:N+1,1:N+1}$, 
is equal to the matrix $S$ containing all singular values, 
\mbox{$S = S_r$}; 
 and the 
the 
matrices $V_{12}, V_{22}$ are
void in this case.
The first $r$ 
vectors in $V$ form a basis for the image of $K^f_{1:N+1,1:N+1}$:
A coordinate transformation
\begin{align}
  \begin{bmatrix}
    v_1 \\
    v_2
  \end{bmatrix}
  &= 
  \begin{bmatrix}
    V_{11}^\top & V_{21}^\top \\
    V_{12}^\top & V_{22}^\top
  \end{bmatrix}
  \alpha^f
  &&
  \Leftrightarrow
  &&
  \alpha^f = 
  \begin{bmatrix}
    V_{11} & V_{12} \\
    V_{21} & V_{22}
  \end{bmatrix}
  \begin{bmatrix}
    v_1 \\
    v_2
  \end{bmatrix}
  \label{eq:svd_kf_transform}
\end{align}
reveals that 
\mbox{$K^f_{1:N+1,1:N+1} \alpha^f = \begin{bmatrix} V_{11} \\ V_{21} \end{bmatrix} S_r v_1$}.
Hence, neither the
optimal cost 
nor the constraints
of~\eqref{eq:maximizer_representer} and~\eqref{eq:sup_findim_opt_app_alpha_relax}
depend on $v_2$,
implying that
there exists an optimal solution which satisfies $v_2 = 0$.

To simplify notation, 
we denote by 
\begin{align}
  \Phi_{1:N+1} &= \begin{bmatrix}
    \Phi_{1:N} \\
    \Phi_{N+1}
  \end{bmatrix} 
  =
  \begin{bmatrix} V_{11} \\ V_{21} \end{bmatrix} S_r^{1/2}
  \label{eq:svd_kf_phi}
\end{align}
the \emph{feature matrix} associated with the kernel matrix 
\begin{align}
   K^f_{1:N+1,1:N+1} &= \Phi_{1:N+1} \Phi_{1:N+1}^\top 
   \label{eq:svd_kf_phiphit}
\end{align}
Defining as $\theta \doteq S_r^{1/2} v_1 \in \mathbb{R}^r$ the corresponding weight vector, it holds that 
\begin{align}
\label{eq:Kfalpha_Phitheta}
K^f_{1:N+1,1:N+1} \alpha^f = \Phi_{1:N+1} \theta.
\end{align}

\subsubsection*{Eliminating the subspace determined by training data}

Interpolation of the training data by the latent function and noise process 
uniquely determines the components of the optimal solution in an $N$-dimensional subspace,
while the remaining orthogonal components are not affected by this constraint.
We find this subspace by 
applying a QR decomposition
\begin{align}
  \label{eq:relaxed_proof_qr}
  \begin{bmatrix}
    \Phi_{1:N}^\top \\
    (\sigma R^w_{1:N,1:N})^\top
  \end{bmatrix}
  &=
  \underbrace{
  \begin{bmatrix}
    Q_{11} & Q_{12} \\
    Q_{21} & Q_{22}
  \end{bmatrix}
  }_{\doteq Q}
  \begin{bmatrix}
    R \\
    0
  \end{bmatrix},
\end{align}
where 
\mbox{$Q \in \mathbb{R}^{(N+r) \times (N+r)}$} is an orthonormal matrix, \mbox{$R \in \mathbb{R}^{N \times N}$} is upper-triangular,
and
\begin{align}
  \sigma^2 K^w_{1:N,1:N} &= \sigma R^w_{1:N,1:N} (\sigma R^w_{1:N,1:N})^\top
  \label{eq:noise_covar_chol}
  \intertext{is the (upper-triangular) Cholesky decomposition of the noise covariance matrix or,
equivalently,}
  (\sigma^2 K^w_{1:N})^{-1} &= \left( \sigma R^w_{1:N,1:N} (\sigma R^w_{1:N,1:N})^\top \right)^{-1} = (\sigma R^w_{1:N,1:N})^{-\top} (\sigma R^w_{1:N,1:N})^{-1}
\end{align}
is the standard (lower-triangular) Cholesky decomposition of the inverse noise covariance matrix.
We use the orthogonal matrix $Q$ from the QR decomposition
to define a coordinate transformation
\begin{align}
  \begin{bmatrix}
    Q_{11} & Q_{12} \\
    Q_{21} & Q_{22}
  \end{bmatrix}
  \begin{bmatrix}
    \delta_1 \\
    \delta_2
  \end{bmatrix}
  &= 
  \begin{bmatrix}
    I_r & 0 \\
    0 & (\sigma R^w_{1:N,1:N})^{-1}
  \end{bmatrix}
  \begin{bmatrix}
    \theta \\
    c_w
  \end{bmatrix},
  \label{eq:delta_vc_trafo}
\end{align}
with $\delta_1 \in \mathbb{R}^N$, 
$\delta_2 \in \mathbb{R}^r$,
which
allows 
compute the 
components of the solution determined by the training data. 
For clarity, we emphasize that the partitioning of the matrices in \cref{eq:delta_vc_trafo} on the left- and right-hand side is different:
the first line on the left-hand side contains $N$~rows, the first line on the right-hand side, $r$~rows.

\newpage

\section{Proof of \cref{thm:relaxed_upper_bound}}
\label{sec:app_relaxed_proof}

Starting from the finite-dimensional formulation of the relaxed problem~\eqref{eq:sup_infdim_relax} as given in~\eqref{eq:sup_findim_opt_app_alpha_relax}, we first apply the two coordinate transformations presented in~\cref{sec:app_coordinate_trafo} (\cref{sec:lemma1_preliminary}). This allows us to obtain a simplified problem formulation --- a linear program with a single norm-ball constraint --- that can be solved directly (\cref{sec:lemma1_analytical}). Finally, we also present the result for the lower bound (\cref{sec:lemma1_lower_bound}).

\subsection{Preliminary coordinate transformation}\label{sec:lemma1_preliminary}

Using the SVD of the kernel matrix~\eqref{eq:svd_kf} as well as the coordinate transformation~\eqref{eq:delta_vc_trafo},
the data equation~\eqref{eq:sup_findim_opt_app_alpha_relax_data} reads
\begin{align}
\label{eq:relaxed_data_equality_constraint_QR}
  y &\stackrel{\eqref{eq:Kfalpha_Phitheta}}{=}
  \Phi_{1:N} \theta + c^w, \notag \\
  &=
  \begin{bmatrix}
    \Phi_{1:N} & \sigma R^w_{1:N,1:N}
  \end{bmatrix}
  \begin{bmatrix}
    I_N & 0 \\
    0 & (\sigma R^w_{1:N,1:N})^{-1}
  \end{bmatrix}
  \begin{bmatrix}
    \theta \\
    c_w
  \end{bmatrix}, \notag \\
  &\stackrel{\eqref{eq:relaxed_proof_qr}}{=}
  \begin{bmatrix}
    R^\top & 0
  \end{bmatrix}
  \begin{bmatrix}
    Q_{11}^\top & Q_{21}^\top \\
    Q_{12}^\top & Q_{22}^\top
  \end{bmatrix}
  \begin{bmatrix}
    I_N & 0 \\
    0 & (\sigma R^w_{1:N,1:N})^{-1}
  \end{bmatrix}
  \begin{bmatrix}
    \theta \\
    c_w
  \end{bmatrix} \notag \\
  &\stackrel{\eqref{eq:delta_vc_trafo}}{=} R^\top \delta_1.
\end{align}
This leads to \mbox{$\delta_1^{\star,\sigma} = R^{-\top} y$} being fully determined by the data, 
leaving only \mbox{$\delta_2 \in \mathbb{R}^r$} to be optimized. 
The RKHS-norm constraint~\eqref{eq:sup_findim_opt_app_alpha_relax_bnd_fw} is reformulated as
\begin{align}
  (\alpha^f)^\top K^f_{1:N+1,1:N+1} \alpha^f + (c^w)^\top (\sigma^2 K^{w}_{1:N,1:N})^{-1} c^w 
  \stackrel{\eqref{eq:svd_kf_phiphit},\eqref{eq:Kfalpha_Phitheta}}{=}&
  \left\| 
  \begin{bmatrix}
    I_N & 0 \\
    0 & (\sigma R^w_{1:N,1:N})^{-1}
  \end{bmatrix}
  \begin{bmatrix}
    \theta \\
    c_w
  \end{bmatrix}
  \right\|_2^2 \notag \\
  \stackrel{\eqref{eq:delta_vc_trafo}}{=}& \| \delta_1 \|_2^2 + \| \delta_2 \|_2^2,
\end{align}
where we have used that $Q$ is orthogonal, i.e., \mbox{$\| Q x \|_2^2 = \| x \|_2^2$} for all \mbox{$x \in \mathbb{R}^{r+N}$}.
Finally, in the new coordinates, the cost is expressed as

\begin{align}
    K^f_{N+1,1:N} \alpha^f \stackrel{\eqref{eq:Kfalpha_Phitheta}}{=} \Phi_{N+1} \theta 
    \stackrel{\eqref{eq:delta_vc_trafo}}{=} \Phi_{N+1} (Q_{11} \delta_1 + Q_{12} \delta_2).
\end{align}

Problem~\eqref{eq:sup_findim_opt_app_alpha_relax} is thus equivalently reformulated as follows:
\begin{subequations}
  \label{eq:sup_findim_relax_app_delta}
  \begin{align}
    \overline{f}^\sigma(x_{N+1}) = \sup_{\substack{
      \delta_2 \in \mathbb{R}^{r} \\
    }} 
    \quad & \Phi_{N+1} Q_{12} \delta_2 + \Phi_{N+1} Q_{11} \delta_1^{\star,\sigma}\\
    \mathrm{s.t.} \quad 
    & \| \delta_2 \|_2^2 \leq \Gamma_f^2 + \frac{\Gamma_w^2}{\sigma^2} - \| \delta_1^{\star,\sigma} \|_2^2.
  \end{align}
\end{subequations}

\subsection{Analytical solution}\label{sec:lemma1_analytical}

With its linear cost and norm-ball constraint, problem~\eqref{eq:sup_findim_relax_app_delta} has the unique optimal solution
\begin{align}
\label{eq:relaxed_proof_sol}
  \delta_2^{\star,\sigma} &= \frac{Q_{12}^\top \Phi_{N+1}^\top}{\| Q_{12}^\top \Phi_{N+1}^\top \|_2} \sqrt{\Gamma_f^2 + \frac{\Gamma_w^2}{\sigma^2} - \| \delta_1^{\star,\sigma} \|_2^2}
\end{align}
and associated optimal cost
\begin{align}
  \label{eq:relaxed_proof_sol_cost}
  \overline{f}^\sigma(x_{N+1}) &= \Phi_{1:N} Q_{11} \delta_1^{\star,\sigma} + \| Q_{12}^\top \Phi_{N+1}^\top \|_2 \sqrt{\Gamma_f^2 + \frac{\Gamma_w^2}{\sigma^2} - \| \delta_1^{\star,\sigma} \|_2^2}
\end{align}

To obtain the formulation in in~\eqref{eq:relaxed_upper_bound}, we use the following relations inferred from the QR decomposition~\eqref{eq:relaxed_proof_qr}:
\begin{subequations}
\begin{align}
  Q_{12} Q_{12}^\top &= I - Q_{11} Q_{11}^\top, \label{eq:relaxed_proof_qr_q1212} \\
  Q_{11} &= \Phi_{1:N}^\top R^{-1}, \label{eq:relaxed_proof_qr_q11} \\
  R^\top R &= R^\top Q^\top Q R \notag = \Phi_{1:N} \Phi_{1:N}^\top + (\sigma R^w_{1:N,1:N})^\top \sigma R^w_{1:N,1:N} \notag \\ &= K^f_{1:N,1:N} + \sigma^2 K^w_{1:N,1:N}. \label{eq:relaxed_proof_qr_rtr}
\end{align}
\end{subequations}
We can now simplify the terms in the optimal cost~\eqref{eq:relaxed_proof_sol_cost}.
First, it holds that
\begin{align*}
  \Phi_{N+1} Q_{11} \delta_1^{\star,\sigma} &\overset{\eqref{eq:relaxed_proof_qr_q11}}{=} (\Phi_{N+1} \Phi_{1:N}^\top) R^{-1} R^{-\top} y, \\
  &\overset{\eqref{eq:relaxed_proof_qr_rtr}}{=} K^f_{N+1,1:N} \left( K^f_{1:N,1:N} + \sigma^2 K^w_{1:N,1:N} \right)^{-1} y, \\
  &\overset{{\eqref{eq:gp_mean}}}{=} \fmu(x_{N+1}).
\intertext{Then, we have}
  \| Q_{12}^\top \Phi_{N+1}^\top  \|_2^2 &\overset{\phantom{\eqref{eq:relaxed_proof_qr_q11}}}{=} \Phi_{N+1} Q_{12} Q_{12}^\top \Phi_{N+1}^\top  \\
  &\overset{\eqref{eq:relaxed_proof_qr_q1212}}{=} \Phi_{N+1} (I - Q_{11} Q_{11}^\top) \Phi_{N+1}^\top  \\
  &\overset{\eqref{eq:relaxed_proof_qr_q11}}{=} \Phi_{N+1} \Phi_{N+1}^\top - ( \Phi_{N+1} \Phi_{1:N}^\top  ) R^{-1} R^{-\top} ( \Phi_{1:N} \Phi_{N+1}^\top  ) \\
  &\overset{\eqref{eq:relaxed_proof_qr_rtr}}{=} K^f_{N+1,N+1} - K^f_{N+1,1:N} \left( K^f_{1:N,1:N} + \sigma^2 K^w_{1:N,1:N} \right)^{-1} K^f_{1:N,N+1} \\
  &\overset{{\eqref{eq:gp_covariance}}}{=} \Sigma^f_\sigma(x_{N+1}).
\intertext{Lastly, we obtain}
  \| \delta_1^{\star,\sigma} \|_2^2 &\overset{{\eqref{eq:relaxed_data_equality_constraint_QR}}}{=} y^\top R^{-1} R^{-\top} y \\
  &\overset{\eqref{eq:relaxed_proof_qr_rtr}}{=} y^\top \left( K^f_{1:N,1:N} + \sigma^2 K^w_{1:N,1:N} \right)^{-1} y \\
  &\overset{\eqref{eq:minimum_norm_interpolant_sum_spaces}}{=} \| \gmu \|_{\Hksum}^2.
\end{align*}
To summarize, this shows that the optimal cost of~\eqref{eq:sup_infdim_relax} is given by
\begin{align}
  \label{eq:relaxed_proof_opt_cost}
  \overline{f}^\sigma(x_{N+1}) &= \fmu(x_{N+1}) + \sqrt{\Gamma_f^2 + \frac{\Gamma_w^2}{\sigma^2} - \| \gmu \|_{\Hksum}^2} \sqrt{\Sigma^f_\sigma(x_{N+1})}.\notag
\end{align}

\subsection{Optimal relaxed solution for the lower bound}\label{sec:lemma1_lower_bound}

For the lower bound, 
the same derivations apply with a minor change.
Flipping the sign in the cost leads leads to a flipped sign in the optimal solution for the free variables $\delta_2$,
i.e., \mbox{$\delta_2^{\star,\inf} = - \delta_2^{\star,\sigma}$}. 
This results in the optimal cost for the lower bound
\begin{align}
  \underline{f}^\sigma(x_{N+1}) &= \fmu(x_{N+1}) - \sqrt{\Gamma_f^2 + \frac{\Gamma_w^2}{\sigma^2} - \| \gmu \|_{\Hksum}^2} \sqrt{\Sigma^f_\sigma(x_{N+1})}.\notag
\end{align}

Due to the symmetry of the relaxed bounds around~$\fmu(x_{N+1})$, the following corollary is immediate.
\begin{corollary}
    Let \cref{ass:RKHS_norm_fw,ass:distinct_input_locations} be satisfied. Then, for all $\sigma \in (0,\infty)$, it holds that
    \begin{align*}
        | \ftr(x_{N+1}) - \fmu(x_{N+1}) | \leq \sqrt{\Gamma_f^2 + \frac{\Gamma_w^2}{\sigma^2} - \| \gmu \|_{\Hksum}^2} \sqrt{\Sigma^f_\sigma(x_{N+1})}.
    \end{align*}
\end{corollary}

\newpage

\section{Proof of \cref{thm:optimal_solution}}
\label{sec:app_optimal_proof}

In the following, we derive an analytic solution to Problem~\eqref{eq:sup_infdim_opt}. Taking its finite-dimensional formulation~\eqref{eq:sup_findim_opt_app_alpha}, we eliminate the noise coefficients as a function of the latent function coefficients, \mbox{$c^w = y - \Phi_{1:N} \theta$}, and deploy~\eqref{eq:svd_kf_phi} to obtain the following reformulation:

\begin{subequations}
  \label{eq:sup_findim_opt_app}
  \begin{align}
    \overline{f}(x_{N+1}) = \sup_{\substack{
      \theta \in \mathbb{R}^{r} %
    }} 
    \quad & \Phi_{N+1} \theta \\
    \mathrm{s.t.} \quad 
    & \theta^\top \theta - \Gamma_f^2 \leq 0, \label{eq:sup_findim_opt_app_rkhs_f} \\
    & (y -\Phi_{1:N} \theta)^\top (K^{w}_{1:N,1:N})^{-1} (y - \Phi_{1:N} \theta) - \Gamma_w^2 \leq 0. \label{eq:sup_findim_opt_app_rkhs_w}
  \end{align}
\end{subequations}

We will analyze the solution of Problem~\eqref{eq:sup_findim_opt_app} 
for different active sets.
Here, the term 
``active set''
refers to a subset of the RKHS-norm constraints~\eqref{eq:sup_findim_opt_app_rkhs_f} and~\eqref{eq:sup_findim_opt_app_rkhs_w} that 
are strictly active,
i.e., influence the optimal primal solution of the problem.
For strictly active constraints~\eqref{eq:sup_findim_opt_app_rkhs_f} and~\eqref{eq:sup_findim_opt_app_rkhs_w}
there exist
respective Lagrangian multipliers,
$\lambda^f$ and $\lambda^w$, that are strictly positive. 
We investigate the following combinations:
\begin{enumerate}
    \item[Case 1]: only \eqref{eq:sup_findim_opt_app_rkhs_f} is strictly active (\mbox{$\lambda^f > 0$}, \mbox{$\lambda^w = 0$}), 
    \item[Case 2]: only \eqref{eq:sup_findim_opt_app_rkhs_w} is strictly active (\mbox{$\lambda^f = 0$}, \mbox{$\lambda^w > 0$}), 
    \item[Case 3]: both~\eqref{eq:sup_findim_opt_app_rkhs_f} and \eqref{eq:sup_findim_opt_app_rkhs_w} are strictly active (\mbox{$\lambda^f > 0$}, \mbox{$\lambda^w > 0$}).
    \item[Case 4]: both~\eqref{eq:sup_findim_opt_app_rkhs_f} and \eqref{eq:sup_findim_opt_app_rkhs_w} are not strictly active (\mbox{$\lambda^f = 0$}, \mbox{$\lambda^w = 0$}),
\end{enumerate}
Based on the solutions for fixed active sets%
\footnote{
The presented analysis of Cases~1-3 considers a slightly more permissive setting, which allows some of the Lagrange multipliers to be zero, i.e., the corresponding constraint to be inactive or weakly active. Thus, in some scenarios multiple cases might be applicable (see \cref{sec:thm1_active_set}); yet, this does not affect our analysis as the cases cover all possible scenarios. %
}%
, the optimal solution can then be found by case distinction. 
We discuss each case separately, obtaining the corresponding analytical solution, presenting the feasibility check and elucidating the connection with the solution of the relaxed problem in \cref{sec:case1,sec:case2,sec:case3,sec:case4}; note that \cref{sec:case2,sec:case3} provide the proofs for~\cref{thm:optimal_solution_case1,thm:optimal_solution_case2,thm:optimal_solution_case2_corollary_testtrain}.  
We then show how to practically check which set is active, and obtain the desired claim~\eqref{eq:optimal_solution_inf} in~\cref{sec:thm1_active_set}. We conclude the section by presenting the result for the lower bound (\cref{sec:thm1_lower_bound}).

\subsection{Case 1: Noise constraint inactive}\label{sec:case1}

We now consider the case in which only~\eqref{eq:sup_findim_opt_app_rkhs_f} is active, proving~\cref{thm:optimal_solution_case1}.

\paragraph{Optimal solution }
Problem~\eqref{eq:sup_findim_opt_app} with omitted constraint~\eqref{eq:sup_findim_opt_app_rkhs_w} 
is given by
\begin{subequations}
  \label{eq:sup_findim_opt_app_case1_nodata}
  \begin{align}
    \overline{f}_1(x_{N+1}) = \sup_{\substack{
      \theta \in \mathbb{R}^{r}
    }} 
    \quad & \Phi_{N+1} \theta \\
    \mathrm{s.t.} \quad 
    & \theta^\top \theta - \Gamma_f^2 \leq 0. \label{eq:sup_findim_opt_app_case1_nodata_rkhs_f}
  \end{align}
\end{subequations}
The solution of the above optimization problem is given as
\begin{align}
  \theta^{\star,1} &= \frac{\Phi_{N+1}^\top }{\| \Phi_{N+1} \|_2} \Gamma_f,
  \label{eq:eq:sup_findim_opt_app_case1_thetaopt}
\end{align}
which results in the optimal cost given in~\eqref{eq:cost_proposition_1}, namely
\begin{align}
  \overline{f}_1(x_{N+1}) \overset{\phantom{\eqref{eq:svd_kf_phi}}}{=} \frac{\Phi_{N+1}\Phi_{N+1}^{\top}}{\| \Phi_{N+1}\|_2} 
  \Gamma_f
   \overset{\eqref{eq:svd_kf_phi}}{=} \sqrt{K^f_{N+1,N+1}} \Gamma_f.\notag
\end{align}

\paragraph{Feasibility check}
The optimizer~$\theta^{\star,1}$ 
is a feasible solution of~\eqref{eq:sup_findim_opt_app} if the
corresponding optimal noise coefficients
\begin{align}
  c^{w,\star,1} \doteq y - \frac{\Phi_{1:N} \Phi_{N+1}^\top }{\sqrt{\Phi_{N+1} \Phi_{N+1}^\top}} \Gamma_f \stackrel{\eqref{eq:svd_kf_phi}}{=} y - K^f_{1:N,N+1} \frac{\Gamma_f}{\sqrt{K^f_{N+1,N+1}}},\notag
\end{align}
satisfy the neglected constraint~\eqref{eq:sup_findim_opt_app_rkhs_w}, 
i.e.,
if
\begin{align}
  \left\| y - K^f_{1:N,N+1} \frac{\Gamma_f}{\sqrt{K^f_{N+1,N+1}}} \right\|_{(K^w_{1:N,1:N})^{-1}}^2 \leq \Gamma_w^2,\notag
\end{align}
as given in~\eqref{eq:constraint_proposition1}).

\paragraph{Connection to relaxed solution}
For \mbox{$\sigma \rightarrow \infty$}, for the relaxed solution~$\overline{f}^\sigma(x_{N+1})$ in \cref{sec:app_relaxed_proof}, it holds that
\begin{align}
  \lim_{\sigma \rightarrow \infty} \fmu &= \lim_{\sigma \rightarrow \infty} K^f_{N+1,1:N} \left( K^f_{1:N,1:N} + \sigma^2 K^w_{1:N,1:N} \right)^{-1} y \notag \\
  &= 0, \notag \\
  \lim_{\sigma \rightarrow \infty} \beta_\sigma &= \lim_{\sigma \rightarrow \infty} \sqrt{\Gamma_f^2 + \frac{\Gamma_w^2}{\sigma^2} - \| \gmu \|_{\Hksum}^2 } \notag \\
  &= \Gamma_f, \notag\\
  \lim_{\sigma \rightarrow \infty} \Sigma^f_\sigma(x_{N+1}) &= \lim_{\sigma \rightarrow \infty} K^f_{N+1,N+1} - K^f_{N+1,1:N} \left( K^f_{1:N,1:N} + \sigma^2 K^w_{1:N,1:N} \right)^{-1} K^f_{1:N,N+1} \notag \\
  &= K^f_{N+1,N+1}.\notag
\end{align}
Thus, the relaxed solution converges to the optimal solution for \mbox{$\sigma \rightarrow \infty$}, i.e.,
\begin{align}
  \lim_{\sigma \rightarrow \infty} \overline{f}^\sigma(x_{N+1}) &= \lim_{\sigma \rightarrow \infty} \fmu + \beta_\sigma \sqrt{\Sigma^f_\sigma(x_{N+1})} \notag \\
  &= \sqrt{K^f_{N+1,N+1}} \Gamma_f \notag \\
  &= \overline{f}_1(x_{N+1}). \notag
\end{align}

\subsection{Case 2: Function constraint inactive}\label{sec:case2}

We proceed by considering the case in which only~\eqref{eq:sup_findim_opt_app_rkhs_f} is active, proving the result given in~\cref{thm:optimal_solution_case2}.

\paragraph{Optimal solution} 
Problem~\eqref{eq:sup_findim_opt_app} under this active set is given as
\begin{subequations}
  \label{eq:sup_findim_opt_app_case2}
  \begin{align}
    \overline{f}_2(x_{N+1}) = \sup_{\substack{
      \theta \in \mathbb{R}^{r}
    }} 
    \quad & \Phi_{N+1} \theta \label{eq:sup_findim_opt_app_case2_cost} \\
    \mathrm{s.t.} \quad 
    & (y - \Phi_{1:N} \theta)^\top (K^{w}_{1:N,1:N})^{-1} (y - \Phi_{1:N} \theta) - \Gamma_w^2 \leq 0. \label{eq:sup_findim_opt_app_case2_rkhs_w}
  \end{align}
\end{subequations}
This optimization problem only has a finite optimal cost if the span of 
$\Phi_{N+1}^\top \in \mathbb{R}^{r\times 1}$ 
is contained in the span of 
\mbox{$\Phi_{1:N}^\top \in \mathbb{R}^{r \times N}$}, i.e., if 
\mbox{$\mathrm{span}(\Phi_{N+1}^\top) \subseteq \mathrm{span}(\Phi_{1:N}^\top)$}.
Otherwise, there would exist a direction \mbox{$d_1 \in \mathrm{span}(\Phi_{N+1}^\top)$} such that \mbox{$\Phi_{1:N} d_1 = 0$}: 
 the optimal solution to~\eqref{eq:sup_findim_opt_app_case2} would then be unbounded and thus would not satisfy the constraint~\eqref{eq:sup_findim_opt_app_rkhs_f} of the original problem. 
Hence, in the following, we focus on the case where 
\mbox{$\mathrm{span}(\Phi_{N+1}^\top) \subseteq \mathrm{span}(\Phi_{1:N}^\top)$}.

If \mbox{$\mathrm{span}(\Phi_{N+1}^\top) \subseteq \mathrm{span}(\Phi_{1:N}^\top)$}, we can write %
\mbox{$\Phi_{N+1}^\top$} as a linear combination of the column vectors of \mbox{$\Phi_{1:N}^\top$}, i.e., \mbox{$\Phi_{N+1}^\top = \Phi_{1:N}^\top \lambda$}, where \mbox{$\lambda \in \mathbb{R}^{N \times 1}$}. 
Since the 
$r$~feature
vectors in 
\begin{align}
  \Phi_{1:N+1}
  =
  \begin{bmatrix}
    \Phi_{1:N} \\ \Phi_{N+1}
  \end{bmatrix}
  = 
  \begin{bmatrix}
    I_N \\
    \lambda^\top
  \end{bmatrix}
  \Phi_{1:N}
\end{align}
are linearly independent,
$\Phi_{1:N}$ has full column rank.
As $\Phi_{1:N}^\top$ thus has full row rank, $\lambda$ can be determined as \mbox{$\lambda = \Phi_{1:N} (\Phi_{1:N}^\top \Phi_{1:N})^{-1} \Phi_{N+1}^\top$}.

To reformulate constraint~\eqref{eq:sup_findim_opt_app_case2_rkhs_w} as a norm-ball constraint,
we employ a QR decomposition.
Recalling the upper-triangular Cholesky factor \mbox{$R^w_{1:N,1:N} (R^w_{1:N,1:N})^\top = K^w_{1:N,1:N}$} of the noise covariance matrix from~\eqref{eq:noise_covar_chol},
we factor the matrix
\begin{align}
  (R^w_{1:N,1:N})^{-1} \Phi_{1:N} = 
  \underbrace{%
  \begin{bmatrix}
    \tilde{Q}_{11} & \tilde{Q}_{12} \\
    \tilde{Q}_{21} & \tilde{Q}_{22}
  \end{bmatrix}
  }_{\doteq \tilde{Q}}
  \begin{bmatrix}
    \tilde{R} \\
    0 
  \end{bmatrix},
  \label{eq:sup_findim_opt_app_case2_qr}
\end{align}
to obtain an orthonormal matrix \mbox{$\tilde{Q} \in \mathbb{R}^{N \times N}$},
with \mbox{$\tilde{Q}^\top \tilde{Q} = I$},
and an upper-triangular matrix \mbox{$\tilde{R} \in \mathbb{R}^{r \times r}$}.
The QR factorization implies
the following 
relations
required for the proof:
\begin{subequations}
\begin{align}
  \tilde{Q} \tilde{Q}^\top = I 
  \label{eq:sup_findim_opt_app_case2_qtq}
  \\
  \begin{bmatrix}
    \tilde{Q}_{11} \\
    \tilde{Q}_{21}
  \end{bmatrix}
  &= 
  (R^w_{1:N,1:N})^{-1} 
  \Phi_{1:N}
  \tilde{R}^{-1},
  \label{eq:sup_findim_opt_app_case2_q1121}
  \\
  \tilde{R}^\top \tilde{R} &= \tilde{R}^\top \tilde{Q}^\top \tilde{Q} \tilde{R} = \Phi_{1:N}^\top (K^w_{1:N,1:N})^{-1} \Phi_{1:N}, 
  \label{eq:sup_findim_opt_app_case2_rtr}
  \\
  \begin{bmatrix}
    \tilde{Q}_{12} \\ \tilde{Q}_{22}
  \end{bmatrix}
  \begin{bmatrix}
    \tilde{Q}_{12}^\top & \tilde{Q}_{22}^\top 
  \end{bmatrix}
  &= 
  I - 
  \begin{bmatrix}
    \tilde{Q}_{11} \\ \tilde{Q}_{21}
  \end{bmatrix}
  \begin{bmatrix}
    \tilde{Q}_{11}^\top & \tilde{Q}_{21}^\top 
  \end{bmatrix}.
  \label{eq:sup_findim_opt_app_case2_qouter}
\end{align}
\end{subequations}

This
allows to write the 
constraint~\eqref{eq:sup_findim_opt_app_case2_rkhs_w} as
\begin{align*}
  \Gamma_w^2
  &\geq 
  \left\| (R^w_{1:N,1:N})^{-1} (y - \Phi_{1:N} \theta) \right\|_2^2 \\
  &\overset{\eqref{eq:sup_findim_opt_app_case2_qr}}{=} 
  \left\| 
  (R^w_{1:N,1:N})^{-1} y
  - 
  \tilde{Q}
  \begin{bmatrix}
    \tilde{R} \\
    0 
  \end{bmatrix}
  \theta
  \right\|_2^2 \\
  &\overset{\eqref{eq:sup_findim_opt_app_case2_qtq}}{=} 
  \left\| 
  \tilde{Q}^\top
  (R^w_{1:N,1:N})^{-1} y
  - 
  \begin{bmatrix}
    \tilde{R} \\
    0 
  \end{bmatrix}
  \theta
  \right\|_2^2 \\
  &\overset{\eqref{eq:sup_findim_opt_app_case2_qouter}}{=} 
  \left\|
  \begin{bmatrix}
    \tilde{Q}_{11}^\top & \tilde{Q}_{21}^\top 
  \end{bmatrix}
  (R^w_{1:N,1:N})^{-1} y
  -
  \tilde{R} 
  \theta
  \right\|_2^2 
  \\
  &\phantom{=}
  +
  \left\| (R^w_{1:N,1:N})^{-1} y \right\|_2^2
  - 
  \left\|
    \begin{bmatrix}
      \tilde{Q}_{11}^\top & \tilde{Q}_{21}^\top 
    \end{bmatrix}
    (R^w_{1:N,1:N})^{-1} y
  \right\|_2^2 
  \\
  &\overset{\eqref{eq:sup_findim_opt_app_case2_q1121}}{=} 
  \left\|
  \tilde{R}^{-\top} \Phi_{1:N}^\top 
  (K^w_{1:N,1:N})^{-1} y
  -
  \tilde{R} 
  \theta
  \right\|_2^2 
  \\
  &\phantom{=}
  +
  \left\| (R^w_{1:N,1:N})^{-1} y \right\|_2^2
  - 
  \left\|
    \tilde{R}^{-\top} \Phi_{1:N}^\top 
    (K^w_{1:N,1:N})^{-1} y
  \right\|_2^2 
  \\
  &\overset{\eqref{eq:sup_findim_opt_app_case2_rtr}}{=}
    \| z_1 \|_2^2 + \tilde{y}^\top \left( K^w_{1:N,1:N} - M \right) \tilde{y},
\end{align*}
where we used the following definitions in the last line:
\begin{align}
  z_1 &\doteq
  \tilde{R}^{-\top} \Phi_{1:N}^\top 
  (K^w_{1:N,1:N})^{-1} y
  -
  \tilde{R} 
  \theta, 
  \label{eq:sup_findim_opt_app_case2_z1}
  \\
  \tilde{y} &\doteq (K^{w}_{1:N,1:N})^{-1} y, 
  \label{eq:sup_findim_opt_app_case2_definitions_ytilde}
  \\
  M &\doteq \Phi_{1:N} \left( \Phi_{1:N}^\top (K^w_{1:N,1:N})^{-1} \Phi_{1:N} \right)^{-1} \Phi_{1:N}^\top.
  \label{eq:sup_findim_opt_app_case2_definitions_M}
\end{align}
Using the coordinate transformation~\eqref{eq:sup_findim_opt_app_case2_z1} and \mbox{$\Phi_{N+1}^\top = \Phi_{1:N}^\top \lambda$}, the cost~\eqref{eq:sup_findim_opt_app_case2_cost} is rewritten as
\begin{align}
  \lambda^\top \Phi_{1:N} \theta &= 
  \lambda^\top \Phi_{1:N}
  \tilde{R}^{-1} 
  \left( 
    \tilde{R}^{-\top} \Phi_{1:N}^\top 
  (K^w_{1:N,1:N})^{-1} y - z_1 
  \right) \notag \\
  &= \lambda^\top M \tilde{y} - 
  \lambda^\top 
  \Phi_{1:N}
  \tilde{R}^{-1} z_1,
\end{align}
leading to the formulation of~\eqref{eq:sup_findim_opt_app_case2} in the transformed coordinates:
\begin{subequations}
  \label{eq:sup_findim_opt_app_case2_z}
  \begin{align}
    \overline{f}_2(x_{N+1}) = \sup_{\substack{
      z_1 \in \mathbb{R}^{r}
    }} 
    \quad &
    \lambda^\top M \tilde{y} - \lambda^\top \Phi_{1:N} \tilde{R}^{-1} z_1
    \\
    \mathrm{s.t.} \quad 
    & \| z_1 \|_2^2 \leq \left( \Gamma_w^2
    - \tilde{y}^\top \left( K^w_{1:N,1:N} - M \right) \tilde{y} \right).
    \label{eq:sup_findim_opt_app_case2_z_rkhs_w}
  \end{align}
\end{subequations}
Noting that, by \cref{ass:RKHS_norm_fw}, the right-hand side of the constraint~\eqref{eq:sup_findim_opt_app_case2_z_rkhs_w} is non-negative, the optimal solution of the above problem is given as
\begin{align}
  z_1^\star &= 
  - 
  \frac{
    \tilde{R}^{-\top} \Phi_{1:N}^\top \lambda
  }{
    \|
    \tilde{R}^{-\top} \Phi_{1:N}^\top \lambda
    \|_2
  }
  \sqrt{
    \Gamma_w^2 - \tilde{y}^\top \left( K^w_{1:N,1:N} - M \right) \tilde{y}
  },
  \label{eq:sup_findim_opt_app_case2_zstar}
\end{align}
leading to the corresponding optimal cost~\eqref{eq:optimal_solution_case2_general}:
\begin{align}
  \overline{f}_2(x_{N+1}) &= \sqrt{\lambda^\top M \lambda} \sqrt{
    \Gamma_w^2 - \tilde{y}^\top \left( K^w_{1:N,1:N} - M \right) \tilde{y}
  } + \lambda^\top M \tilde{y},
  \label{eq:sup_findim_opt_app_case2_opt_cost} 
  \\
  &\stackrel{\eqref{eq:sup_findim_opt_app_case2_rtr},\eqref{eq:sup_findim_opt_app_case2_definitions_M}}{=} \| P \Phi_{N+1}^\top \|_{P^{-1}} \sqrt{
    \Gamma_w^2 - y^\top \left( K^w_{1:N,1:N} \right)^{-1} y + \| \theta^\mu \|_{P^{-1}}^2} + \Phi_{N+1} \theta^\mu,
  \label{eq:sup_findim_opt_app_case2_opt_cost_fogel} 
\end{align}
where \eqref{eq:sup_findim_opt_app_case2_opt_cost_fogel} follows by
utilizing that 
\mbox{$\Phi_{1:N}^\top \lambda = \Phi_{N+1}^\top$}
as well as by 
defining 
the weighting matrix
\begin{align*}
  P &\doteq \left( \Phi_{1:N}^\top \left( K^w_{1:N,1:N} \right)^{-1} \Phi_{1:N} \right)^{-1}
\end{align*}
and 
the least-squares estimator for the unknown parameters
\begin{align*}
  \theta^\mu &\doteq P \Phi_{1:N}^\top \left( K^w_{1:N,1:N} \right)^{-1} y.
\end{align*}

\paragraph{Feasibility check} 
In the original coordinates, the optimal solution is given as
\begin{align}
  \label{eq:sup_findim_opt_app_case2_opt_v1}
  \theta^{\star,2} &= \tilde{R}^{-1} \tilde{R}^{-\top} \Phi_{1:N}^\top 
  (K^w_{1:N,1:N})^{-1} y - \tilde{R}^{-1} z_1^\star \notag \\
  &\stackrel{\eqref{eq:sup_findim_opt_app_case2_rtr}}{=} \left( \Phi_{1:N}^\top (K^w_{1:N,1:N})^{-1} \Phi_{1:N} \right)^{-1} \Phi_{1:N}^\top \left( \tilde{y} 
  + 
  \lambda \frac{\sqrt{
    \Gamma_w^2 - \tilde{y}^\top \left( K^w_{1:N,1:N} - M \right) \tilde{y}
  }}{\sqrt{\lambda^\top M \lambda}} \right) \notag \\
  &= \thetamu[] + \frac{P \Phi_{N+1}^\top}{\| P \Phi_{N+1}^\top \|_{P^{-1}}} \sqrt{\Gamma_w^2 - y^\top \left( K^w_{1:N,1:N} \right)^{-1} y + \| \theta^\mu \|_{P^{-1}}^2};
\end{align}
the point $\theta^{\star,2}$ is feasible for the original problem~\eqref{eq:sup_findim_opt_app}
if it satisfies the neglected constraint~\eqref{eq:sup_findim_opt_app_rkhs_f}, i.e., if
\mbox{$\| \theta^{\star,2} \|_2^2 \leq \Gamma_f^2$}, retrieving~\eqref{eq:optimal_solution_case2_feasibility_b}.

\paragraph{Connection to relaxed solution}
The quantities in the optimal cost~\eqref{eq:sup_findim_opt_app_case2_opt_cost} 
can be expressed as limiting values related to the relaxed solution for \mbox{$\sigma \rightarrow 0$}. 
For the matrix $M$ in~\eqref{eq:sup_findim_opt_app_case2_definitions_M}, it holds that
\begin{align}
  M &= \Phi_{1:N} \left( \Phi_{1:N}^\top (K^w_{1:N,1:N})^{-1} \Phi_{1:N} \right)^{-1} \Phi_{1:N}^\top \notag \\
  &= \lim_{\sigma \rightarrow 0} \Phi_{1:N} \left( \sigma^2 I_r + \Phi_{1:N}^\top (K^w_{1:N,1:N})^{-1} \Phi_{1:N} \right)^{-1} \Phi_{1:N}^\top \notag \\
  &= \lim_{\sigma \rightarrow 0} \frac{1}{\sigma^2} \Phi_{1:N} \left( I_r - \Phi_{1:N}^\top \left( (\sigma^2 K^w_{1:N,1:N}) + \Phi_{1:N} \Phi_{1:N}^\top \right)^{-1} \Phi_{1:N}^\top \right) \Phi_{1:N}^\top \notag \\
  &= \lim_{\sigma \rightarrow 0} \frac{1}{\sigma^2} \left( \Phi_{1:N} \Phi_{1:N}^\top - \Phi_{1:N} \Phi_{1:N}^\top \left( (\sigma^2 K^w_{1:N,1:N}) + \Phi_{1:N} \Phi_{1:N}^\top \right)^{-1} \Phi_{1:N} \Phi_{1:N}^\top \right). \notag 
\end{align}
With \mbox{$\lambda^\top \Phi_{1:N} = \Phi_{N+1}$}, this results in
\begin{align}
  \lambda^\top M \lambda &= \lim_{\sigma \rightarrow 0} \frac{1}{\sigma^2} \left(  \Phi_{N+1} \Phi_{N+1}^\top - \Phi_{N+1} \Phi_{1:N}^\top  \left( (\sigma^2 K^w_{1:N,1:N}) + \Phi_{1:N} \Phi_{1:N}^\top \right)^{-1} \Phi_{1:N} \Phi_{N+1}^\top \right) \notag \\
  &= \lim_{\sigma \rightarrow 0} \frac{1}{\sigma^2} \left( K^f_{N+1,N+1} - K^f_{N+1,1:N} \left( (\sigma^2 K^w_{1:N,1:N}) + K^f_{1:N,1:N} \right)^{-1} K^f_{1:N,N+1} \right) \notag \\
  & \stackrel{\eqref{eq:gp_covariance}}{=} \lim_{\sigma \rightarrow 0} \frac{1}{\sigma^2} \Sigma^f_\sigma(x_{N+1}). 
\end{align}
The offset \mbox{$\lambda^\top M \tilde{y}$} in \eqref{eq:sup_findim_opt_app_case2_opt_cost} is equivalent to
the offset~$\fmu(x_{N+1})$ in the optimal cost~\eqref{eq:relaxed_proof_opt_cost} for the relaxed problem:
\begin{align}
  \lambda^\top M \tilde{y} &= \lim_{\sigma \rightarrow 0} \frac{1}{\sigma^2} \left(  \Phi_{N+1} \Phi_{1:N}^\top  - \Phi_{N+1} \Phi_{1:N}^\top  \left( (\sigma^2 K^w_{1:N,1:N}) + \Phi_{1:N} \Phi_{1:N}^\top \right)^{-1} \Phi_{1:N} \Phi_{1:N}^\top \right) (K^w_{1:N,1:N})^{-1} y \notag \\
  &= \lim_{\sigma \rightarrow 0} {\Phi_{N+1} \Phi_{1:N}^\top  \left( (\sigma^2 K^w_{1:N,1:N}) + \Phi_{1:N} \Phi_{1:N}^\top \right)^{-1} \left( \Phi_{1:N} \Phi_{1:N}^\top - \Phi_{1:N} \Phi_{1:N}^\top + \sigma^2 K^w_{1:N,1:N} \right) (\sigma^2 K^w_{1:N,1:N})^{-1} y } \notag \\
  &= \lim_{\sigma \rightarrow 0} K^f_{N+1,1:N} \left( (\sigma^2 K^w_{1:N,1:N}) + K^f_{1:N,1:N} \right)^{-1} y \notag \\
  &\stackrel{\eqref{eq:gp_mean}}{=} \lim_{\sigma \rightarrow 0} \fmu(x_{N+1})
\end{align}

For the last term, \mbox{$\Gamma_w^2 - \tilde{y}^\top \left( K^w_{1:N,1:N} - M \right) \tilde{y}$}, 
we have that
\begin{align}
  \Gamma_w^2 - \tilde{y}^\top \left( K^w_{1:N,1:N} - M \right) \tilde{y} &= \lim_{\sigma \rightarrow 0} \sigma^2 \Gamma_f^2 + \Gamma_w^2 - \tilde{y}^\top \left( K^w_{1:N,1:N} - M \right) \tilde{y} \notag \\
  &= \lim_{\sigma \rightarrow 0} \sigma^2 \left( \Gamma_f^2 + \frac{\Gamma_w^2}{\sigma^2} - \frac{1}{\sigma^2} \tilde{y}^\top \left( K^w_{1:N,1:N} - M \right) \tilde{y} \right) \notag \\
  &= \lim_{\sigma \rightarrow 0} \sigma^2 \left( \Gamma_f^2 + \frac{\Gamma_w^2}{\sigma^2} - y^\top \left( W^{-1} - W^{-1} (\sigma^2 M) W^{-1} \right) y \right), \label{eq:proof_case2_last_term_cost}
\end{align}
where we have introduced the abbreviation \mbox{$W \doteq \sigma^2 K^w_{1:N,1:N}$} for notational simplicity;
similarly, we abbreviate \mbox{$F \doteq K^f_{1:N,1:N}$}.
Recalling that \mbox{$M = \frac{1}{\sigma^2}(F - F(W+F)^{-1}F)$} and by using~\cite[Exercise 16.(d), Chapter~5]{searle_matrix_2017},
the data-dependent term in the above expression can be simplified as follows:
\begin{align}
   &  W^{-1} - W^{-1} \left( F - F \left( W + F \right)^{-1} F \right) W^{-1} \notag \\
  = &  W^{-1} - W^{-1} F W^{-1} + W^{-1} F W^{-1} \left( W^{-1} + W^{-1} F W^{-1} \right)^{-1} W^{-1} F W^{-1} \notag \\
  = & W^{-1} - (W^{-1} - W^{-1}(W^{-1}FW^{-1} + W^{-1})^{-1}W^{-1}) %
  \notag \\
  = &  W^{-1}(W^{-1} + W^{-1}FW^{-1})
^{-1}W^{-1} = (W+F)^{-1}.\notag \end{align}

To summarize, it holds that
\begin{align}
  \Gamma_w^2 - \tilde{y}^\top \left( K^w_{1:N,1:N} - M \right) \tilde{y} &= \lim_{\sigma \rightarrow 0} \sigma^2 \Gamma_f^2 + \Gamma_w^2 - \tilde{y}^\top \left( K^w_{1:N,1:N} - M \right) \tilde{y} \notag \\
  &= \lim_{\sigma \rightarrow 0} \sigma^2 \left( \Gamma_f^2 + \frac{\Gamma_w^2}{\sigma^2} - y^\top \left( K^f_{1:N,1:N} + (\sigma^2 K^w_{1:N,1:N}) \right)^{-1} y \right) \notag \\
  &= \lim_{\sigma \rightarrow 0} \sigma^2 \left( \Gamma_f^2 + \frac{\Gamma_w^2}{\sigma^2} - \| \gmu \|_{\Hksum}^2 \right) \notag
\end{align}
and the total bound for Case 2 is given as
\begin{align}
  \overline{f}_2(x_{N+1}) &= \sqrt{\lambda^\top M \lambda} \sqrt{
    \Gamma_w^2 - \tilde{y}^\top \left( K^w_{1:N,1:N} - M \right) \tilde{y}
  } + \lambda^\top M \tilde{y} \notag \\
  &= \sqrt{ \lim_{\sigma \rightarrow 0} \frac{1}{\sigma^2} \Sigma^f_\sigma(x_{N+1}) } \sqrt{ \lim_{\sigma \rightarrow 0} \sigma^2 \left( \Gamma_f^2 + \frac{\Gamma_w^2}{\sigma^2} - \| \gmu \|_{\Hksum}^2 \right)} + \lim_{\sigma \rightarrow 0} \fmu \notag \\
  &= \lim_{\sigma \rightarrow 0} \left( \sqrt{\Sigma^f_\sigma(x_{N+1})} \sqrt{\Gamma_f^2 + \frac{\Gamma_w^2}{\sigma^2} - \| \gmu \|_{\Hksum}^2} + \fmu \right) \notag \\
  &= \lim_{\sigma \rightarrow 0} \overline{f}^\sigma(x_{N+1}).\notag
\end{align}

\subsubsection*{Proof of \cref{thm:optimal_solution_case2_corollary_testtrain}}

We now prove 
\cref{thm:optimal_solution_case2_corollary_testtrain},
which 
simplifies the general result of \cref{thm:optimal_solution_case2}
under the assumptions that
the kernel matrix $K^f_{1:N,1:N}$ 
is
invertible
and
the test point 
is 
equal to the $k$-th training point, i.e., \mbox{$x_{N+1} = x_k$}, for some \mbox{$k \in \{1, \ldots, N\}$}. 
In this case, 
the $k$-th and $(N+1)$-th row of the kernel matrix \mbox{$K^f_{1:N+1,1:N+1}$}
are identical.
In terms of the singular value decomposition, by \eqref{eq:svd_kf} this implies that
\begin{align*}
    K^f_{k,1:N} &= e_k^\top \Phi_{1:N} \Phi_{1:N}^\top
    = \Phi_{N+1} \Phi_{1:N}^\top 
    = K^f_{N+1,1:N},
\end{align*}
i.e.,
the relation 
\mbox{$\Phi_{N+1}^\top = \Phi_{1:N}^\top \lambda$} 
holds for 
\mbox{$\lambda = e_k$},
with $e_k$ being the $k$-th unit vector.
Since \mbox{$\Phi_{1:N} \in \mathbb{R}^{r \times N}$} has rank
\mbox{$r = N$}, 
it 
is 
invertible.
This allows to simplify the expression for the optimal cost using that \mbox{$M \tilde{y} = y$},
with $M$ and $\tilde{y}$ defined as in \cref{eq:sup_findim_opt_app_case2_definitions_ytilde,eq:sup_findim_opt_app_case2_definitions_M}:
\begin{align}
  \overline{f}_2(x_{k}) &= \sqrt{e_k^\top M e_k} \sqrt{
    \Gamma_w^2
  } + e_k^\top y \notag \\
  &= \sqrt{e_k^\top \Phi_{1:N} \left( \Phi_{1:N}^\top (K^w_{1:N,1:N})^{-1} \Phi_{1:N} \right)^{-1} \Phi_{1:N}^\top e_k} \Gamma_w + e_k^\top y \notag \\
  &= \sqrt{e_k^\top K^w_{1:N,1:N} e_k} \Gamma_w + e_k^\top y \notag \\
  &= \sqrt{K^w_{k,k}} \Gamma_w + y_k, \notag
\end{align}
which is the optimal cost as given in~\eqref{eq:constraint_proposition1}.
By inserting the simplified expressions into~\eqref{eq:sup_findim_opt_app_case2_opt_v1}, the optimal \mbox{$\theta$} becomes
\begin{align}
  \theta^{\star,2} 
  &= \left( \Phi_{1:N}^\top (K^w_{1:N,1:N})^{-1} \Phi_{1:N} \right)^{-1} \Phi_{1:N}^\top \left( \tilde{y} 
  +
  e_k \frac{\sqrt{
    \Gamma_w^2
  }}{\sqrt{e_k^\top M e_k}} \right) \notag \\
  &= \Phi_{1:N}^{-1} \left( y 
  + 
  K^w_{1:N,k}
  \frac{\Gamma_w}{\sqrt{K^w_{k,k}}} \right). \notag
\end{align}
Recalling the 
low-rank factorization of~$K^f_{1:N,1:N}$ in \eqref{eq:svd_kf_phi},
the feasibility condition based on the neglected constraint~\eqref{eq:sup_findim_opt_app_rkhs_f} reduces to
\begin{align}
  \| \theta^{\star,2} \|_2^2 
  &= 
  \left\| 
  \left(
  y 
  + 
  K^w_{1:N,k} 
  \frac{\Gamma_w}{\sqrt{K^w_{k,k}}} 
  \right)
  \right\|%
  _{(K^f_{1:N,1:N})^{-1}}^2 \leq \Gamma_f^2, \notag
\end{align}
as presented in~\eqref{eq:optimal_solution_case2_feasibility_a}.

\begin{remark}
    \cref{thm:optimal_solution_case2_corollary_testtrain} has been derived as a particular case of~\cref{thm:optimal_solution_case2}, which provides the analytic solution for the case $\sigma\to 0$ occurring when the kernel matrix is rank-deficient. From this perspective, the scenario in which the test-input belongs to the training data-set is one of the particular situations leading to a drop in the rank of the kernel matrix. However, the proof of~\cref{thm:optimal_solution_case2_corollary_testtrain} could be alternatively carried out following the steps of the one of~\cref{thm:optimal_solution_case1}.  
\end{remark}

\subsection{Case 3: Both constraints active}\label{sec:case3}

Next, we consider the case when both constraints~\eqref{eq:sup_findim_opt_app_rkhs_f} and~~\eqref{eq:sup_findim_opt_app_rkhs_w} are active.

\paragraph{Optimal solution}
We first show that strong duality holds for both the relaxed problem~\eqref{eq:sup_findim_opt_app_alpha_relax} as well as the original problem~\eqref{eq:sup_findim_opt_app}.
Afterwards, we establish that there exists a value \mbox{$\sigma \in (0,\infty)$}, such that the primal optimizer of the relaxed problem is a primal optimizer for the original problem.

For the original problem~\eqref{eq:sup_findim_opt_app}, we show that a strictly feasible solution can be constructed using the true latent function and noise process.
Let \mbox{$\ftrint \in \Hkf$} be the minimum-norm interpolant of the latent function at the 
test and 
training input locations, i.e.,
\begin{subequations}
  \begin{align}
    \ftrint = \underset{f \in \Hkf}{\arg \min} \quad & \| f \|_{\Hkf}^2 \\
    \mathrm{s.t.} \quad & f(x_i) = \ftr(x_i), \> i=1,\ldots,N+1.
  \end{align}
\end{subequations}
Similarly, let \mbox{$\wtrint \in \Hkw$} be the minimum-norm interpolant of the noise-generating process at the training input locations (excluding the test point),
\begin{subequations}
  \begin{align}
    \wtrint = \underset{w \in \Hkw}{\arg \min} \quad & \| w \|_{\Hkw}^2 \\
    \mathrm{s.t.} \quad & w(x_i) = \wtr(x_i), \> i=1,\ldots,N.
  \end{align}
\end{subequations}
The representer theorem~\citep{kimeldorf_results_1971} establishes that the solutions to the above optimization problems is finite-dimensional and given by
\begin{align}
  \ftrint(\cdot) = \sum_{i=1}^{N+1} \kf(\cdot, x_i) \alpha^{f,\mathrm{tr}}_i, \quad 
  \wtrint(\cdot) = \sum_{i=1}^{N} \kw(\cdot, x_i) \alpha^{w,\mathrm{tr}}_i.\notag
\end{align}
By design, the sum of both functions interpolates the training data, i.e., \mbox{$\ftrint(x_i) + \wtrint(x_i) = y_i$} for \mbox{$i = 1, \ldots, N$}. 
Additionally, by \cref{ass:RKHS_norm_fw}, it holds that $\ftrint$ and $\wtrint$ satisfy their corresponding RKHS-norm bound, i.e., \mbox{$\|\ftrint\|_{\Hkf}^2\leq \| \ftr \|_{\Hkf}^2 < \Gamma_f^2$} and \mbox{$\|\wtrint\|_{\Hkw}^2\leq \| \wtr \|_{\Hkw}^2 < \Gamma_w^2$}.
Thus, the corresponding coefficient vector 
\mbox{$\thetatr \stackrel{\eqref{eq:svd_kf_phi}}{=} S_r^{1/2} v_1^{\mathrm{tr}} = S_r^{1/2} \begin{bmatrix}
  V_{11}^\top & V_{12}^\top
\end{bmatrix} \alpha^{f,\mathrm{tr}}$}
constitutes a strictly feasible solution of the finite-dimensional problem formulation~\eqref{eq:sup_findim_opt_app}.
This implies that Slater's condition is satisfied for the convex program~\eqref{eq:sup_findim_opt_app},
which
implies that strong duality holds.
Since every strictly feasible solution for the original problem~\eqref{eq:sup_findim_opt_app} is also strictly feasible for 
the relaxed problem~\eqref{eq:sup_findim_relax_app_delta}, 
similarly, 
Slater's condition and strong duality hold for the relaxed problem.

Due to strong duality, 
the point 
$\theta^{\star,\sigma}$
is the unique minimizer of the relaxed problem~\eqref{eq:sup_findim_opt_app_alpha_relax}
if and only if 
the primal-dual pair 
$(\theta^{\star,\sigma}, \lambda_g^{\star,\sigma})$
satisfies the KKT conditions
\begin{subequations}
\begin{align}
  \nabla_{\theta} \mathcal{L}_\sigma(\theta^{\star,\sigma}, \lambda_g^{\star,\sigma}) &= 0, \\
  \left( \| \theta^{\star,\sigma} \|_2^2 - \Gamma_f^2 \right) + \frac{1}{\sigma^2} \left( \| y - \Phi_{1:N} \theta^{\star,\sigma} \|_{(K^{w}_{1:N,1:N})^{-1}}^2 - \Gamma_w^2 \right) &\leq 0, \\
  \lambda_g^{\star,\sigma} \left( \left( \| \theta^{\star,\sigma} \|_2^2 - \Gamma_f^2 \right) + \frac{1}{\sigma^2} \left( \| y - \Phi_{1:N} \theta^{\star,\sigma} \|_{(K^{w}_{1:N,1:N})^{-1}}^2 - \Gamma_w^2 \right) \right) &= 0, \\
  \lambda_g^{\star,\sigma} &\geq 0,
\end{align}
\label{eq:KKT_relaxed}
\end{subequations}
with the corresponding Lagrangian
\begin{align*}
  \mathcal{L}_\sigma(\theta, \lambda_g) &= \Phi_{N+1} \theta - \lambda_g \left( \| \theta \|_2^2 - \Gamma_f^2 + \frac{1}{\sigma^2} \left( \| y - \Phi_{1:N} \theta \|_{(K^{w}_{1:N,1:N})^{-1}}^2 - \Gamma_w^2 \right) \right) \\
  &= \Phi_{N+1} \theta - \lambda_g \left( \| \theta \|_2^2 - \Gamma_f^2 \right)  - \frac{\lambda_g}{\sigma^2} \left( \| y - \Phi_{1:N} \theta \|_{(K^{w}_{1:N,1:N})^{-1}}^2 - \Gamma_w^2 \right).
\end{align*}

Similarly, due to strong duality, the point 
$\theta^{\star}$
is the unique minimizer if the original problem~\eqref{eq:sup_findim_opt_app} if and only if the primal-dual pair 
\mbox{$(\theta^\star, \lambda_f^\star, \lambda_w^{\star,\sigma})$} 
satisfies the KKT conditions
\begin{align*}
  \nabla_{\theta} \mathcal{L}(\theta^\star, \lambda_f^\star, \lambda_w^{\star,\sigma}) &= 0 \\
  \| \theta^\star \|_2^2 - \Gamma_f^2 &\leq 0 \\
 \| y - \Phi_{1:N} \theta^\star \|_{(K^{w}_{1:N,1:N})^{-1}}^2 - \Gamma_w^2 &\leq 0 \\
  \lambda_f^\star \left( \| \theta^\star \|_2^2 - \Gamma_f^2 \right) &= 0 \\
  \lambda_w^{\star,\sigma} \left( \| y - \Phi_{1:N} \theta^\star \|_{(K^{w}_{1:N,1:N})^{-1}}^2 - \Gamma_w^2 \right) &= 0 \\
  \lambda_f^\star, \lambda_w^{\star,\sigma} &\geq 0
\end{align*}
with corresponding Lagrangian
\begin{align*}
  \mathcal{L}(\theta, \lambda_f, \lambda_w) &= \Phi_{N+1} \theta - \lambda_f \left( \| \theta \|_2^2 - \Gamma_f^2 \right) - \lambda_w \left( \| y - \Phi_{1:N} \theta \|_{(K^{w}_{1:N,1:N})^{-1}} - \Gamma_w^2 \right).
\end{align*}
Let 
$(\theta^{\star,3}, \lambda_f^{\star,3}, \lambda_w^{\star,3})$ 
be the 
optimal primal-dual 
solution satisfying the KKT conditions of the original problem~\eqref{eq:sup_findim_opt_app}
under the imposed active set.
Since 
the 
constraints~\eqref{eq:sup_findim_opt_app_rkhs_f} and~\eqref{eq:sup_findim_opt_app_rkhs_w} 
are
active, 
it holds that \mbox{$\lambda_f^{\star,3}, \lambda_w^{\star,3} > 0$}. 
Now, let \mbox{$(\sigma^\star)^2 = \frac{\lambda_f^{\star,3}}{\lambda_w^{\star,3}}$}.
Then, \mbox{$\lambda_f^{\star,3} = (\sigma^\star)^2 \lambda_w^{\star,3}$} and
the primal-dual pair 
\mbox{$(\theta^{\star,3}, \lambda_f^{\star,3})$} 
satisfy the KKT conditions~\eqref{eq:KKT_relaxed} of the relaxed problem: 
\begin{enumerate}
  \item Since 
  \mbox{$\nabla_{\theta} \mathcal{L}_\sigma(\theta^{\star,3}, \lambda_f^{\star,3}) = \nabla_{\theta} \mathcal{L}(\theta^{\star,3}, \lambda_f^{\star,3}, \lambda_w^{\star,3}) = 0$}, 
  the stationarity condition is fulfilled.
  \item As both constraints~\eqref{eq:sup_findim_opt_app_rkhs_f} and~\eqref{eq:sup_findim_opt_app_rkhs_w} are active,
  \begin{align}
    \| \theta^{\star,3} \|_2^2 - \Gamma_f^2 = -\left( \| y - \Phi_{1:N} \theta^{\star,3} \|_{(K^{w}_{1:N,1:N})^{-1}}^2 - \Gamma_w^2 \right) = 0,\notag 
  \end{align} 
  i.e., primal feasibility and complementarity slackness are fulfilled.
  \item The optimal multiplier \mbox{$\lambda_g^{\star,\sigma} = \lambda_f^{\star,3} > 0$} for the relaxed problem is positive. 
\end{enumerate}
Hence, 
due to strong duality
since both constraints~\eqref{eq:sup_findim_opt_app_rkhs_f} and~\eqref{eq:sup_findim_opt_app_rkhs_w} are active, 
\mbox{$(\theta^{\star,3}, \lambda_f^{\star,3})$} 
is the optimal primal-dual solution for the relaxed problem~\eqref{eq:sup_findim_opt_app_alpha_relax} with \mbox{$\sigma = \sigma^\star = \sqrt{\lambda_f^{\star,3} / \lambda_w^{\star,3}}$} 
if and only if 
\mbox{$(\theta^{\star,3}, \lambda_f^{\star,3}, \lambda_w^{\star,3})$} 
is the optimal primal-dual solution for the original problem~\eqref{eq:sup_findim_opt_app}.

\paragraph{Feasibility check} 

The solution is feasible by definition.

\paragraph{Connection to relaxed solution}

As shown above, the optimal cost can be recovered by the cost of the relaxed problem for a specific choice of noise parameter $\sigma = \sigma^\star$, i.e.,
\begin{align}
  \overline{f}_3(x_{N+1}) = \overline{f}^{\sigma^\star}(x_{N+1}).
\end{align}

\subsection{Case 4: Both constraints inactive}
\label{sec:case4}
Last, we investigate the case when both constraints~\eqref{eq:sup_findim_opt_app_rkhs_f} and \eqref{eq:sup_findim_opt_app_rkhs_w} are inactive. 

\paragraph{Optimal solution}

The optimal solution to the unconstrained linear program
is given by case distinction:
\begin{align}
  \overline{f}_4(x_{N+1}) &= 
  \sup_{\substack{
      \theta \in \mathbb{R}^{r}
    }} 
    \quad \Phi_{N+1} \theta 
    \label{eq:sup_findim_opt_app_case4}
    \\
  &= \begin{cases}
    0, &\text{if $\Phi_{N+1} = 0$,} \\
    \infty, &\text{if $\Phi_{N+1} \neq 0$.}
  \end{cases}
  \notag
\end{align}

\paragraph{Feasibility check}

For $\Phi_{N+1} \neq 0$, as the optimal solution is unbounded,
it is infeasible 
for the original problem~\eqref{eq:sup_findim_opt_app}, 
whose feasible set is compact, in particular due to constraint~\eqref{eq:sup_findim_opt_app_rkhs_f}. 
Hence, this case never corresponds to the optimal solution. 

If $\Phi_{N+1} = 0$, 
any $\theta^{\star,4} \in \mathbb{R}^{r}$ is optimal. 
Thus, 
any
feasible solution satisfying constraints~\eqref{eq:sup_findim_opt_app_rkhs_f} and~\eqref{eq:sup_findim_opt_app_rkhs_w} is also optimal.

\paragraph{Connection to relaxed solution}

If $\Phi_{N+1} = 0$, the solution of the relaxed problem~\eqref{eq:sup_findim_opt_app_alpha_relax} is given by $\overline{f}^\sigma(x_{N+1}) = 0$, 
i.e.,
it recovers the solution of the original problem~\eqref{eq:sup_findim_opt_app_case4} for any $\sigma \in (0, \infty)$.

\subsection{Finding the correct active set}\label{sec:thm1_active_set}

Let \mbox{$\theta^{\star,i}$} denote the primal solution corresponding to Case \mbox{$i$}, with \mbox{$i=1,\dots,4$}. 
The active set for the optimal solution is given by the one for which the corresponding primal solution $\theta^{\star,i}$ is feasible for the original problem and leads to the maximum cost among all feasible optimizers $\theta^{\star,i}$ for a specific active set, i.e., 
\begin{subequations}
\begin{align}
  \overline{f}(x_{N+1}) = \max_{ i \in \{ 1,2,3,4 \} } \quad & \Phi_{N+1} \theta^{\star,i} \\
  \mathrm{s.t.} \quad & \| \theta^{\star,i} \|_2^2 \leq \Gamma_f^2, \\
  & \| y - \Phi_{1:N} \theta^{\star,i} \|_{(K^w_{1:N,1:N})^{-1}}^2 \leq \Gamma_w^2. 
\end{align}
\end{subequations}
The solution to the above problem can be obtained by case distinction.
The optimal cost for 
a subset of active constraints
lower-bounds
the optimal cost for 
a superset,
i.e., \mbox{$\Phi_{N+1} \theta^{\star,4} \geq \Phi_{N+1} \theta^{\star,j} \geq \Phi_{N+1} \theta^{\star,3}$} for \mbox{$j \in \{1,2\}$}. Thus, if $\theta^{\star,4}$ is feasible, it will be optimal. 
Otherwise, if either $\theta^{\star,1}$ or $\theta^{\star,2}$ is feasible, it will be optimal.
If neither of the other cases is feasible, $\theta^{\star,3}$ is the optimal solution.

Now, we compare the optimal cost if
both $\theta^{\star,1}$ and $\theta^{\star,2}$ are feasible.
If $\theta^{\star,1}$ is a feasible solution of~\eqref{eq:sup_findim_opt_app}, then it holds that the neglected constraint~\eqref{eq:sup_findim_opt_app_rkhs_w} 
does not change the optimal solution of~\eqref{eq:sup_findim_opt_app}.
Similarly, if $\theta^{\star,1}$ is a feasible solution of~\eqref{eq:sup_findim_opt_app}, then it holds that the neglected constraint~\eqref{eq:sup_findim_opt_app_rkhs_f} 
does not change the optimal solution of~\eqref{eq:sup_findim_opt_app}.
Combining both facts, it holds that
\begin{align}
  \Phi_{N+1} \theta^{\star,1} = \Phi_{N+1} \theta^{\star} = \Phi_{N+1} \theta^{\star,2},\notag
\end{align}
i.e., the optimal cost in Case 1 and Case 2 is equal, \mbox{$\overline{f}_1(x_{N+1}) = \overline{f}_2(x_{N+1})$}.
Finally, we note that in Case 4 it holds that $\Phi_{N+1} \theta^{\star} = \Phi_{N+1} \theta^{\star,4} = 0 = \Phi_{N+1} \theta^{\star,j}$, $j \in \{1, 2, 3\}$, i.e., the optimal cost in all four cases is equal. Therefore, it does not need to be considered explicitly.

To summarize, the optimal cost 
is determined as follows:
\begin{align}
  \overline{f}(x_{N+1}) =
  \begin{cases}
    \sqrt{K^f_{N+1,N+1}} \Gamma_f, &\text{%
      in Cases~1 and 4,
    } \\
  \Phi_{N+1} \theta^\mu + \| P \Phi_{N+1}^\top \|_{P^{-1}} \sqrt{
    \Gamma_w^2 - y^\top \left( K^w_{1:N,1:N} \right)^{-1} y + \| \theta^\mu \|_{P^{-1}}^2}
  &\text{%
    in Cases 2  and 4,
  } \\
    \fmu[\sigma^\star](x_{N+1}) + \sqrt{\Gamma_f^2 + \frac{\Gamma_w^2}{(\sigma^\star)^2} - \| \gmu[\sigma^\star] \|_{\mathcal{H}_{\kf + k^{\sigma^\star}}}^2} \sqrt{\Sigma^f_{\sigma^\star}(x_{N+1})},
    &\text{in Cases 3 and 4.} 
  \end{cases} \notag \\
  =
  \begin{cases}
    \lim_{\sigma \rightarrow \infty} \overline{f}^\sigma(x_{N+1}), &\text{if 
    \mbox{$\left\| y - K^f_{1:N,N+1} \frac{\Gamma_f}{\sqrt{K^f_{N+1,N+1}}} \right\|_{(K^w_{1:N,1:N})^{-1}}^2 \leq \Gamma_w^2$}
    }, \\
    \lim_{\sigma \rightarrow 0} \overline{f}^\sigma(x_{N+1}), &\text{if 
    \mbox{$\left\| \thetamu[] + \frac{P \Phi_{N+1}^\top}{\| P \Phi_{N+1}^\top \|_{P^{-1}}} \sqrt{\Gamma_w^2 - y^\top \left( K^w_{1:N,1:N} \right)^{-1} y + \| \theta^\mu \|_{P^{-1}}^2} \right\|_2^2 \leq \Gamma_f^2$},
    } \\
    \inf_{\sigma \in (0,\infty)} \overline{f}^\sigma(x_{N+1}), &\text{otherwise.}
  \end{cases}
  \notag
\end{align}

Finally, 
we show that the analytical solutions in all cases can be reduced to a single expression:
\begin{enumerate}
  \item 
    Let
    Case 1 
    be
    feasible, i.e., 
    \mbox{$\overline{f}(x_{N+1}) = \overline{f}_1(x_{N+1}) = \lim_{\sigma \rightarrow \infty} \overline{f}^\sigma(x_{N+1})$}. Since \mbox{$\overline{f}_1(x_{N+1}) \leq \overline{f}_3(x_{N+1})$}, it holds that 
    \mbox{$\lim_{\sigma \rightarrow \infty} \overline{f}^\sigma(x_{N+1}) \leq \inf_{\sigma \in (0,\infty)} \overline{f}^\sigma(x_{N+1})$}.
    However, for the infimum it also holds that \mbox{$\inf_{\sigma \in (0,\infty)} \overline{f}^\sigma(x_{N+1}) \leq \lim_{\sigma \rightarrow \infty} \overline{f}^\sigma(x_{N+1})$}. Therefore, it holds that \mbox{$\overline{f}(x_{N+1}) = \lim_{\sigma \rightarrow \infty} \overline{f}^\sigma(x_{N+1}) = \inf_{\sigma \in (0,\infty)} \overline{f}^\sigma(x_{N+1})$}.
  \item 
    Let
    Case 2 
    be
    feasible, i.e., 
    \mbox{$\overline{f}(x_{N+1}) = \overline{f}_2(x_{N+1}) = \lim_{\sigma \rightarrow 0} \overline{f}^\sigma(x_{N+1})$}. Analogously as above, since 
    \mbox{$\overline{f}_2(x_{N+1}) \leq \overline{f}_3(x_{N+1})$}
    and
    \mbox{$\inf_{\sigma \in (0,\infty)} \overline{f}^\sigma(x_{N+1}) \leq \lim_{\sigma \rightarrow 0} \overline{f}^\sigma(x_{N+1})$},
    it holds that \mbox{$\overline{f}(x_{N+1}) = \lim_{\sigma \rightarrow 0} \overline{f}^\sigma(x_{N+1}) = \inf_{\sigma \in (0,\infty)} \overline{f}^\sigma(x_{N+1})$}.
  \item 
    In Case~3, since it holds that $\overline{f}_3(x_{N+1}) = \overline{f}^{\sigma}(x_{N+1})$ for a specific value of $\sigma = \sigma^\star \in (0, \infty)$,
    this implies
    that \mbox{$\inf_{\sigma \in (0,\infty)} \overline{f}^\sigma(x_{N+1}) \leq \overline{f}_3(x_{N+1})$}.
    However, 
    since any feasible solution \mbox{$\theta \in \mathbb{R}^r$} of the original problem~\eqref{eq:sup_findim_opt_app} is also feasible for the relaxed problem~\eqref{eq:sup_findim_opt_app_alpha_relax} for any \mbox{$\sigma \in (0, \infty)$},
    the cost of the original problem is upper-bounded by the cost of the relaxed problem, 
    i.e., it also holds that \mbox{$\overline{f}_3(x_{N+1}) \leq \inf_{\sigma \in (0,\infty)} \overline{f}^\sigma(x_{N+1})$}.
    Combining both inequalities, it 
    thus 
    follows 
    that 
    \mbox{$\overline{f}_3(x_{N+1}) = \inf_{\sigma \in (0,\infty)} \overline{f}^\sigma(x_{N+1})$}.
\end{enumerate}
Therefore, 
we have shown that
\begin{align}
  \overline{f}(x_{N+1}) &= \inf_{\sigma \in (0, \infty)} \overline{f}^\sigma(x_{N+1}),\notag
\end{align}
as claimed in~\eqref{eq:optimal_solution_inf}.

\subsection{Lower bound}
\label{sec:thm1_lower_bound} 

The lower bound corresponding to \cref{thm:optimal_solution} 
is obtained by the same steps as the upper bound, 
replacing ``$\sup$'' with ``$\inf$''.
In Cases 1 and 2, 
this leads to a flipped sign 
for the solutions of the free components in \cref{eq:eq:sup_findim_opt_app_case1_thetaopt,eq:sup_findim_opt_app_case2_zstar},
affecting the feasibility checks. 
Overall, the optimal lower bound is also obtained by case distinction: 
\begin{align}
  \underline{f}(x_{N+1}) =
  \begin{cases}
    -\sqrt{K^f_{N+1,N+1}} \Gamma_f, &\text{%
      in Cases~1 and 4,
    } \\
  \Phi_{N+1} \theta^\mu - \| P \Phi_{N+1}^\top \|_{P^{-1}} \sqrt{
    \Gamma_w^2 - y^\top \left( K^w_{1:N,1:N} \right)^{-1} y + \| \theta^\mu \|_{P^{-1}}^2}
  &\text{%
    in Cases 2  and 4,
  } \\
    \fmu[\sigma^\star](x_{N+1}) - \sqrt{\Gamma_f^2 + \frac{\Gamma_w^2}{(\sigma^\star)^2} - \| \gmu[\sigma^\star] \|_{\mathcal{H}_{\kf + k^{\sigma^\star}}}^2} \sqrt{\Sigma^f_{\sigma^\star}(x_{N+1})},
    &\text{in Cases 3 and 4.} 
  \end{cases} \notag \\
  =
  \begin{cases}
    \lim_{\sigma \rightarrow \infty} \underline{f}^\sigma(x_{N+1}), &\text{if 
    \mbox{$\left\| y + K^f_{1:N,N+1} \frac{\Gamma_f}{\sqrt{K^f_{N+1,N+1}}} \right\|_{(K^w_{1:N,1:N})^{-1}}^2 \leq \Gamma_w^2$}
    }, \\
    \lim_{\sigma \rightarrow 0} \underline{f}^\sigma(x_{N+1}), &\text{if 
    \mbox{$\left\| \thetamu[] - \frac{P \Phi_{N+1}^\top}{\| P \Phi_{N+1}^\top \|_{P^{-1}}} \sqrt{\Gamma_w^2 - y^\top \left( K^w_{1:N,1:N} \right)^{-1} y + \| \theta^\mu \|_{P^{-1}}^2} \right\|_2^2 \leq \Gamma_f^2$},
    } \\
    \sup_{\sigma \in (0,\infty)} \underline{f}^\sigma(x_{N+1}), &\text{otherwise.}
  \end{cases}
  \notag
\end{align}
Analogous to \cref{sec:thm1_active_set}, 
by replacing ``$\inf$'' with ``$\sup$'' and flipping the corresponding inequalities, 
the optimal solution is shown to be given as
\begin{align}
  \underline{f}(x_{N+1}) = \sup_{\sigma \in (0,\infty)} \underline{f}^\sigma(x_{N+1}).
\end{align} 
Note that this bound is \emph{not symmetric}, as the supremum and infimum can be attained for different values of the noise parameter~$\sigma$.

\newpage

\section{Numerical example on safe control for uncertain nonlinear systems: Implementation details}
\label{sec:appendix_CBF_example}

In the following, we describe the setup of the optimization problem solved in the numerical example ``Safe control for uncertain nonlinear systems'' in \cref{sec:discussion_stochastic}.
Inserting the system dynamics~\eqref{eq:CBF_dynamics} into the safety constraint \mbox{$x(k+1) \geq (1 - \gamma) x(k)$}, 
the condition \mbox{$f^{\mathrm{known}}(x(k), u(k)) + \ftr(x(k), u(k)) \geq (1 - \gamma) x(k)$} can be enforced robustly by utilizing the lower uncertainty bound
$\fmu(x,u) - \beta_\sigma \sqrt{\Sigma_\sigma(x,u)} \leq \ftr(x,u)$,
leading to the tightened constraint
\begin{align*}
  (1 - \gamma) x(k) 
  &\leq 
    f^{\mathrm{known}}(x(k), u(k)) + \fmu(x(k),u(k))
  - \beta_\sigma \sqrt{\Sigma_\sigma(x(k),u(k))}.
\end{align*}
Minimization of the user-defined cost $c(x(k),u(k))$ subject to the above 
constraint
can be achieved by solving the following optimization problem:
\begin{align*}
  \min_{x_0, x_1, u_0, \sigma, s} \quad & 
  c(x_0, u_0) + \omega s \\
  \mathrm{s.t.} \quad 
  & x_0 = x(k), \\
  & x_1 = f^{\mathrm{known}}(x_0, u_0) + \fmu(x_0, u_0), \\
  & x_1 \geq (1 - \gamma) x_0 + \beta_\sigma \sqrt{\Sigma_\sigma(x_0,u_0)} - s, \\
  & s \geq 0, \\
  & u_{\max} \geq u_0 \geq u_{\min}.
\end{align*}
Therein, the added slack variable~$s$ allows the optimizer to converge even in case the decrease condition is impossible to satisfy given the bounds~$u_{\min}, u_{\max}$ on the control input.
A large linear penalty~$\omega > 0$ 
thereby incentivizes $s = 0$, i.e., constraint satisfaction;
a solution is classified as feasible if the optimal slack value~$s^\star$ satisfies $s^\star \leq 10^{-6}$.
The optimization problems are implemented in \texttt{CasADi}~\citep{andersson_casadi_2019} and solved using the interior-point optimizer \texttt{IPOPT}~\citep{wachter_implementation_2006} on an \mbox{Intel i9-7940X CPU}.
\cref{tab:app_cbf_params} provides all parameters and expressions used for the implementation.

\begin{table}[h!]
  \centering
  \begin{tabular}{r|c}
    Parameter & Value \\
    \hline
    $f^{\mathrm{known}}(x,u)$ & $0.5x + u - 1$ \\
    $c(x,u)$ & $(f^{\mathrm{known}}(x,u) + \fmu(x, u))^2 + u^2$ \\
    $u_{\min}$ & $-2$ \\
    $u_{\max}$ & $2$ \\
    $\gamma$ & $0.95$ \\
    $\omega$ & $10^4$
  \end{tabular}
  \vspace{1ex}
  \caption{Parameters and expressions for the CBF example}
  \label{tab:app_cbf_params}
\end{table}

Additional implementation details can be found in the published source code at \url{https://gitlab.ethz.ch/ics/bounded-energy-rkhs-bounds} and at \url{https://doi.org/10.3929/ethz-c-000785083}.

\fi

\end{document}